\documentclass[abstract=on, final]{scrartcl}
\DeclareOldFontCommand{\bf}{\normalfont\bfseries}{\mathbf}

\usepackage{amsmath, amsthm, amssymb}
\usepackage{tikz}
\usetikzlibrary{trees,arrows,shapes,positioning}
\tikzset{%
  standard/.style={%
    draw,circle
  },
  goal/.style={%
    draw,fill=green!20,diamond
  },
  goal level/.style={%
    draw,fill=gray!20
  }
}
\usepackage{pgfplots}
\usepackage{subcaption}
\usepackage{hyperref}
\usepackage[capitalise,noabbrev]{cleveref}
\usepackage{natbib}
\usepackage[disable]{todonotes}
\usepackage{microtype}
\usepackage{longtable}
\usepackage{siunitx}
\usepackage{multirow}
\usepackage{algorithm}
\usepackage[noend]{algpseudocode}
\usepackage{aliascnt}
\usepackage{wrapfig}
\usepackage[thinlines]{easytable}
\usepackage{graphicx}

\newcommand*{\E}{\mathbb{E}}
\newcommand*{\Geo}{\mathrm{Geo}}
\newcommand*{\truncGeo}{\mathrm{TruncGeo}}
\newcommand*{\Exp}{\mathrm{Exp}}

\newcommand*{\tdfs}{t_{\mathrm{DFS}}}

\newcommand*{\tdfss}{\tilde t^{\mathrm{DFS}}_{\mathrm{SGL}}}
\newcommand*{\tbfss}{t^{\mathrm{BFS}}_{\mathrm{SGL}}}
\newcommand*{\tdfsm}{\tilde t^{\mathrm{DFS}}_{\mathrm{MGL}}}
\newcommand*{\tbfsm}{t^{\mathrm{BFS}}_{\mathrm{MGL}}}
\newcommand*{\tdfsb}{\tilde t^{\mathrm{DFS}}_{\mathrm{BG}}}
\newcommand*{\tdfsbl}{t^{\mathrm{DFS}}_{\mathrm{BGL}}}
\newcommand*{\tdfsbu}{t^{\mathrm{DFS}}_{\mathrm{BGU}}}
\newcommand*{\tbfsb}{t^{\mathrm{BFS}}_{\mathrm{BG}}}
\newcommand*{\tdfsc}{\tilde t^{\mathrm{DFS}}_{\mathrm{CB}}}
\newcommand*{\tdfscl}{t^{\mathrm{DFS}}_{\mathrm{CBL}}}
\newcommand*{\tdfscu}{t^{\mathrm{DFS}}_{\mathrm{CBU}}}
\newcommand*{\tdfsclb}{t^{{\mathrm{DFS}}}_{\mathrm{CBL}}}
\newcommand*{\tdfscub}{t^{{\mathrm{DFS}}}_{\mathrm{CBU}}}
\newcommand*{\tbfsc}{t^{\mathrm{BFS}}_{\mathrm{CB}}}
\newcommand*{\tdfsf}{\tilde t^{\mathrm{DFS}}_{\mathrm{FG}}}
\newcommand*{\tdfsfl}{t^{\mathrm{DFS}}_{\mathrm{FGL}}}
\newcommand*{\tdfsfu}{t^{\mathrm{DFS}}_{\mathrm{FGU}}}
\newcommand*{\tbfsf}{t^{\mathrm{BFS}}_{\mathrm{FG}}}
\newcommand*{\tc}{\mathrm{tc}}

\newcommand*{\descendants}{\mathrm{descendants}}
\newcommand*{\lvl}{\mathrm{level}}

\newcommand*{\lbg}{L^{\mathrm{BG}}}
\newcommand*{\lfg}{L^{\mathrm{FG}}}
\newcommand*{\p}{{\bf p}}
\newcommand*{\pnp}{{\bf p}^{\text{n-puzzle}}}
\newcommand*{\pnpi}{p^{\text{n-puzzle}}}

\newcommand*{\SetR}{\mathbb{R}}

\newcommand*{\SetN}{\mathbb{N}}

\newtheorem{theorem}{Theorem}
\newtheorem{proposition}[theorem]{Proposition}
\newtheorem{corollary}[theorem]{Corollary}
\newtheorem{lemma}[theorem]{Lemma}
\newtheorem{assumption}[theorem]{Assumption}
\theoremstyle{definition}
\newtheorem{definition}[theorem]{Definition}

\usepackage[cal=boondoxo]{mathalfa}

\newcommand*{\up}{\mathrm{up}}
\newcommand*{\side}{\mathrm{side}}
\newcommand*{\down}{\mathrm{down}}
\newcommand*{\bl}{b}            
\newcommand*{\bg}{\beta}          
\newcommand*{\rextra}{r_{\mathrm{extra}}} 
\newcommand*{\pup}{p_{\mathrm{up}}}       
\newcommand*{\pside}{p_{\mathrm{side}}}       
\newcommand*{\pdown}{p_{\mathrm{down}}}       
\newcommand*{\pdu}{p_{\mathrm{down,up}}}
\newcommand*{\pud}{p_{\mathrm{up,down}}}
\newcommand*{\puu}{p_{\mathrm{up,up}}}
\newcommand*{\pdd}{p_{\mathrm{down,down}}}
\newcommand*{\pus}{p_{\mathrm{up,side}}}
\newcommand*{\pds}{p_{\mathrm{down,side}}}
\newcommand*{\psu}{p_{\mathrm{side,up}}}
\newcommand*{\psd}{p_{\mathrm{side,down}}}
\newcommand*{\pss}{p_{\mathrm{side,side}}}

\newcommand*{\bup}[1][]{\bl_{\mathrm{up}#1}}
\newcommand*{\bside}[1][]{\bl_{\mathrm{side}#1}}
\newcommand*{\bdown}[1][]{\bl_{\mathrm{down}#1}}
\newcommand*{\bdir}[1][]{\bl_{\mathrm{dir}#1}}

\newcommand*{\bgup}[1][]{\bg_{\mathrm{up}#1}}
\newcommand*{\bgside}[1][]{\bg_{\mathrm{side}#1}}
\newcommand*{\bgdown}[1][]{\bg_{\mathrm{down}#1}}
\newcommand*{\bgdir}[1][]{\bg_{\mathrm{dir}#1}}

\newcommand*{\floor}[1]{\lfloor #1 \rfloor}

\newcommand*{\dist}{\mathrm{dist}}

\newcommand*{\Seq}{\mathrm{Seq}}
\newcommand*{\seq}{\mathrm{seq}}
\newcommand*{\dir}{\mathrm{dir}}

\newcommand*{\mh}{\mathrm{mh}}
\newcommand*{\Std}{\mathrm{Std}}
\renewcommand*{\bar}{\overline}

\newenvironment{keywords}%
{\begin{abstract}\noindent}%
{\end{abstract}}

\title{A Topological Approach to Meta-heuristics: Analytical Results on the BFS vs.\
  DFS Algorithm Selection Problem.}
\author{Tom Everitt \\ tom.everitt@anu.edu.au \and Marcus Hutter
  \\marcus.hutter@anu.edu.au }
\publishers{Australian National University}

\begin{document}
\maketitle

\begin{abstract}
Search is a central problem in artificial intelligence, and
breadth-first search (BFS) and depth-first search (DFS) are the two most
fundamental ways to search.
In this paper we derive estimates for average BFS and DFS runtime.
The average runtime estimates can be used to
allocate resources or judge the hardness of a problem.
They can also be used for selecting the best graph representation,
and for selecting the faster algorithm out of BFS and DFS.
They may also form the basis for an analysis of more advanced search
methods.
The paper treats both tree search and graph search.
For tree search, we employ a probabilistic model of goal distribution;
for graph search, the analysis depends on an additional statistic of
path redundancy and average branching factor.
As an application, we use the results to predict BFS and DFS runtime
on two concrete grammar problems and on the N-puzzle.
Experimental verification shows that our analytical
approximations come close to empirical reality.
\end{abstract}



\pagebreak
\tableofcontents




\section{Introduction}

Many problems in artificial intelligence may be viewed as
\emph{search problems}, including planning, learning, problem
solving, and (logical) reasoning.
Search problems can often be formulated as \emph{graph search problems},
and can be solved by exploring a space of
possible solutions in a more or less systematic order
\citep{Russell2010, Edelkamp2012}.
%
%
Information that is useful for deciding how to approach a problem
include:
\begin{itemize}
\item How long is the search expected to take for a given
  graph representation and search method?
\item Which graph representation of the problem
  yields the fastest search?
\item Which algorithm is likely to be the fastest?
\end{itemize}
Such knowledge can be used either by a human controller,
or be incorporated in a meta-algorithm for problem solving.

In this study we analyse
the \emph{expected runtime} of
\emph{breadth-first search (BFS)} and \emph{depth-first search (DFS)}.
We focus on expected (or average) runtime,
since expected performance often is the most
relevant measure when allocating resources,
and when choosing algorithm and graph representation.
We focus on BFS and DFS because they are two of the simplest and
most fundamental ways to search, and also exhibit a nice duality
between \emph{searching near} (BFS) and \emph{searching far} (DFS).
Understanding the basic mechanisms of search is likely
to be helpful both in the construction of new search algorithms,
and in the analysis of existing ones.

Previous results on BFS and DFS
have mainly focused on worst case analysis.
For DFS, \citet{Knuth1975} developed an influential technique
for estimating the size of the search tree.
Assuming the tree had similar branching factor in all branches,
\citeauthor{Knuth1975} estimated the search tree size by
multiplying the observed branching factors on the way down through
the tree.
Despite its simplicity, the technique was practically useful
and was subsequently extended and refined by
\citet{Purdom1978}, \citet{Chen1992}, and \citet{Lelis2013}.
Results relevant to BFS include the
analysis of A* \citep{Nilsson1971}
and the analysis of iteratively deepening A* (IDA*)
developed by \citet{Korf2001} and extended by \citet{Zahavi2010}.
When no heuristic information is available A* reduces to
BFS, and IDA* to a memory efficient but slow version of BFS.
Approaches to algorithm selection \citep{Rice1975} have mostly relied on
machine learning techniques applied to problem features.
Such results often provide limited insight into
\emph{why} a certain approach works better in a certain instance
\citep{Kotthoff2014, Hutter2014,Thompson2011, ArbelaezRodriguez2011}.

To facilitate our analysis, we use a probabilistic model of goal
distribution.
%
Our main contribution is an analysis of expected BFS and DFS runtime
as a function of tree depth, goal level, branching factor, and path
redundancy (\cref{sec:sgl,sec:mgl,sec:cb,sec:grammar}).
Estimation of the required parameters is discussed in \cref{sec:graph-param}.
We analyse both tree search and graph search versions of
BFS and DFS.
Following an informal overview of the results in \cref{sec:informal}
and a broader literature review in \cref{sec:litrev},
technical background and definitions are provided in \cref{sec:preliminaries}.
Our analytical results are verified experimentally in
\cref{sec:experimental-verification}.
Conclusions and outlook come in \cref{sec:conclusions}.
Finally, a list of notation can be found in \cref{sec:list-of-notation}.

Some of the results have previously been published in conference
papers \citep{Everitt2015a, Everitt2015b}. In this paper, we
have added sections on estimation of the graph parameters
and on extensions to heuristic search
(\cref{sec:graph-param,sec:heuristic}),
extended the empirical verification (\cref{sec:experimental-verification}),
and made substantial improvements to especially
DFS graph search theory (\cref{sec:dfs-cb}).
We also
provide additional background, illustrations, and discussion,
and add a more extensive literature survey along with a
statement of our grander vision for this work.

\subsection{Informal Overview of Results}
\label{sec:informal}



This section gives an informal account of our results.
A wide range of problems may be formulated as search
in a graph of \emph{nodes} and \emph{edges}.
The search starts in a (possibly random) \emph{start node},
with the aim of reaching a \emph{goal node} via the edges.
For example, consider the search for a university schedule.
A schedule is a \emph{goal node} if no student and no professor
is scheduled to be at multiple places at the same,
and no two classes are simultaneously held in the same room.
\emph{Neighbouring schedules (nodes)} are schedules that can be
reached by a single swap of teacher, location or time.
Such schedules are connected by an edge in the search graph.
One way to do the search is to commence at a random or empty schedule,
and progress by local modifications (i.e., jumps across edges),
until a goal schedule is reached.


There is an infinitude of ways to perform such graph searches.
BFS and DFS are two simple, natural strategies.
They are opposites in the sense that BFS focuses the
search as near to the start node as possible,
while DFS goes as far from the origin as possible.
From this description, one may already suspect that
BFS should benefit when goals are located close to the
origin, while DFS benefits when goals are far from the
origin.
Indeed, our results verify this intuition in a variety
of settings.

We define \emph{runtime} as the number of nodes that need to
be explored until a goal is found.
Throughout, we assume that the maximum search depth
(the \emph{radius of search}) is bounded.

In the simplest model that we investigate in \cref{sec:sgl},
all goals are
located at a certain distance from the start node
in a \emph{tree} search space
where each node is reachable through one path only.
We derive \emph{average} or \emph{expected runtime} as a function of
(1) the distance of the goal from the origin
(the \emph{goal level}), and
(2) the frequency of goals at this distance.

Some interesting observations can be made already in this simple
model.
First, the point where DFS overtakes BFS
depends both on the goal probability and the goal level.
When the goal probability is high, the goal level break point is
roughly halfway between the origin and the maximum search depth
in binary trees.
Unsurprisingly, BFS has the advantage when goals are closer to
the origin, and vice versa.
More interestingly, BFS benefits relative to DFS when the goal probability
gets smaller.
Our model makes the relation precise, and predicts e.g.\
whether DFS will benefit from an increase $x$ in goal depth
combined with a decrease $y$ in goal frequency.
Such knowledge may be useful when choosing
between BFS and DFS, in decisions of how to model
a problem as a graph,
and in the construction of novel meta-heuristics.

We relax the assumptions of the single goal level model in two steps.
The model of \cref{sec:mgl} keeps the tree assumption,
but permits goals to be distributed
at multiple levels, with one goal frequency for every level of the tree.
This makes the analysis of DFS more challenging,
and somewhat coarser approximations are required to obtain a closed
form expression.
BFS can still be analysed exactly.
As before, we find that BFS benefits from goals closer
to the origin, and that DFS benefits from goals closer to the
maximum search depth.
This more general model also enables us to investigate the
effect of spreading goals over many different levels compared
to concentrating the goals to a few levels.
Experimentally, we find that BFS benefits from a greater
spread compared to DFS.
The result holds when the spread is \emph{balanced fairly}
around a central goal level.
We consider a spread \emph{fair} when the goal-likelihood of
a node $k$ levels above the central goal level
is the same as that of a node $k$ levels below.

The final relaxation in \cref{sec:cb}
removes also the tree assumption on the search graph.
Non-tree graphs vary widely along dimensions such as
connectedness/path-redundancy and
average number of neighbours.
These aspects are captured for our analysis in a collection
of parameters called the \emph{length-to-depth counters}.
The length-to-depth counters essentially measure how many nodes
are reachable at various combinations of distances from the
origin, and can be derived from standard parameters such as the
branching factors.
We find that knowing the length-to-depth counters (in addition to
the goal probabilities described before) permits us to approximate
expected BFS runtime, and to give upper and lower bounds on
both DFS tree search and DFS graph search expected runtime.
The DFS bounds may be uninformative in sparsely connected graphs,
where the tree models are more informative.
However, the bounds do provide revealing predictions in
more connected graphs, such as the N-Puzzle and certain grammar
problems. 

\section{Grander Vision and Literature Review}
\label{sec:litrev}

\subsection{Grander Vision}

The grander vision for future work is to construct search algorithms
that adapt their search strategy based on problem features.
%
A very wide range of search algorithms have been developed,
each with their own strengths and weaknesses.
Most of them do not adapt to features of the problem.
Instead, it is usually up to the user to select algorithm
and parameters for each problem.
An automation of this task packaged in a generally applicable
search algorithm could save both developing time and improve
performance.
Since search is a very common problem in AI, the benefits could
be substantial.

Schematically, the solving of many search problems involves
(at least) the following phases:
\begin{enumerate}
\item Start with a problem description.
  For example a SAT-formula to satisfy, a map of cities to traverse,
  or an engineering specification of a VLSI chip.
\item Find a suitable graph representation of the problem.
  This involves specifying what a state is, which states are
  connected, and possibly algorithm-specific operations such as
  how states can be combined and how random states can be generated.
\item Decide and execute a traversal of the search graph.
  For example BFS, DFS, A*, Simulated Annealing, or a genetic algorithm
  \citep{Aarts2003}.
\end{enumerate}

Features that could be useful for algorithm selection could be
mined at any of these stages.
For example, a local sample of the search graph could give estimates
of connectedness, chromatic number, and other graph properties.
The initial findings along a search trajectory can be used to
estimate problem size and runtime \citep{Knuth1975,Kilby2006}.
The original description could also be used: for example, the
number of clauses in a SAT-formula \citep{Haim2008}.
However, the much greater diversity of description types may make
it challenging to create a generally applicable search algorithm
that uses features based on this first stage of the problem solving
(compare an engineering specification for a VLSI chip with a
map for a travelling salesman problem).
In contrast, the underlying search graphs are often readily
comparable, so graph features form a natural starting point.
Constraints on computational resources such as memory and
CPU time are also likely to be valuable features.

Several kinds of inferences could potentially be made from
available problem features.
Inferences could be made analytically, for example through
mathematical proofs showing that under certain conditions
one strategy is better than another.
Another option is to apply machine learning techniques to
experimental data on algorithm performance.
The output of the analysis could either be
a classifier specifying which algorithm is better in which context,
or be aimed at runtime estimates as a function of problem features.
Of course, runtime estimates indirectly define a classifier
of best algorithm (pick the fastest).

To put our aim into context, we next review relevant works.


\subsection{Literature Review}
\label{sec:lit-rev}

We divide our review of related work into two parts.
The works in the first part assume that a portfolio of
predefined algorithms is given, and only try to predict which
algorithm in the portfolio is better for which problem.
The second part reviews approaches that try to build new
search policies, possibly using a set of basic algorithms as
building blocks.

\paragraph{Feature-based algorithm selection}
For a given problem and a given \emph{portfolio} of algorithms,
the \emph{algorithm selection problem} asks which algorithm
is best to use \citep{Rice1975, Kotthoff2014, Smith-Miles2014}.
Tightly related is the question of inferring the search time of different
search algorithms on the problem, as this information can be used
to select the fastest algorithm.
Both analytical investigations and machine learning techniques applied
to empirical data have been tried.
The latter is sometimes known as \emph{empirical performance models}.
For example, \citet{Haim2008} approach the SAT problem,
and predict search time and best search policy
based on properties of the given
formula (such as the number and the size of clauses).
The most comprehensive surveys are
given by \citet{Hutter2014} and \citet{Kotthoff2014},
and the PhD theses by \citet{Thompson2011} and \citet{ArbelaezRodriguez2011}.

As mentioned in the introduction,
\citet{Knuth1975} and \citet{Korf2001} have developed analytical
approaches to estimating the size of the search tree.
This gives a worst-case bound for search performance, since at most
we can search the entire tree.
\citet{Kilby2006} generalise \citeauthor{Knuth1975}'s method,
and also use it to select search policy for the SAT problem
based on which search policy has the lowest estimated runtime.


Many other approaches to the algorithm selection problem
instead try to infer the best search
policy directly, without the intermediate step of estimating runtime.
\citet{Fink1998} does this for STRIPS-like learning using only the
problem size to infer which method is likely to be more efficient.
Schemes using much wider ranges of problem properties are applied to
CSPs by \citet{Thompson2011, ArbelaezRodriguez2011}, and to
the NP-complete problems SAT, TSP and Mixed integer programming
by \citet{Hutter2014}.
\citet{Smith-Miles2012} review and discuss commonly used features
for the algorithm selection problem, mainly applied to the local search scenario.
They divide features into two main categories: General and problem-specific.
General features usually phrased in terms of the \emph{fitness landscape}
(i.e., the target function and the neighbourhood structure).
A common fitness landscape feature is for example the variability
(\emph{ruggedness}) of the target function.
Another general feature is the performance of a simple,
fast algorithm such as gradient descent.
Problem-specific features are discussed for a range of NP-complete
problems such as TSP and Bin-packing.

\paragraph{Constructing a search policy}
There are also meta-approaches to search that do not rely
on a portfolio pre-defined algorithms.
One early example is explanation-based Learning (EBL)
\citep{Dejong1986, Mitchell1986, Minton1988},
which is a general method for learning from examples and domain knowledge.
In the context of search, the domain knowledge is the neighbourhood
function (or the consequence of applying an `action' to a state).
An example to learn from can be the search trace of an algorithm that
has already tried to solve the problem.
The EBL learner analyses the different decisions represented in
the search trace, judges whether they were good or bad, and tries
to find the \emph{reason} they were good or bad. Once a reason
has been found, the gained understanding can be used to pick similar good
decisions at an earlier point during the next search,
and to avoid similar bad decisions
(decisions leading to paths where no goal will be found).
EBL systems have been applied to STRIPS-like planning scenarios
\citep{Minton1988, Minton1990}.

One characteristic feature of EBL is that it requires only one or
a few training examples (in addition to the domain knowledge).
While attractive, it can also lead to overspecific learning \citep{Minton1988}.
Partial Evaluation (PE) is an alternative learning method that is
more robust in this respect, with less dependency on examples \citep{Etzioni1993}.
\citet{Leckie1998} develop an inductive way to learn search control
knowledge (in contrast to the deductive generalisations performed by EBL and PE),
where plenty of training examples substitute for domain knowledge.

A more modern approach is known as \emph{hyper heuristics} \citep{Burke2003, Burke2013}.
It views the problem of inferring good search policies
more abstractly. Rather than interacting with the neighbourhood
structure/graph problem directly, the hyper heuristic only has access
to a set of search policies for the original graph problem.
The search policies are known as \emph{low-level heuristics} in
this literature (not to be confused with \emph{heuristic functions}).
The goal of the hyper heuristic is to find a good policy for
when to apply which low-level heuristic.
Hyper heuristic approaches differs from algorithm selection
in that a new choice of low-level algorithm is made repeatedly,
rather than just once initially.

One example of a hyper heuristic was constructed by \citet{Ross2002}, who
used Genetic Algorithms to learn which
bin-packing heuristic to apply in which type of state in a bin-packing problem.
The learned hyper heuristic outperformed all the provided low-level
heuristics used by themselves.
In applications of hyper heuristics, the low-level heuristics are typically
simple search policies provided
by the human programmers, although nothing prevents them from
being arbitrarily advanced meta-heuristics.
Some research is also being done on automatic construction
of low-level heuristics (see \citep{Burke2013} for references).

Other work on choosing between heuristics include
\citet{domshlak2012,thayer2011,tolpin2013,tolpin2014}.
A related approach directed at programming in general is
\emph{programming by optimisation} \citep{Hoos2012},
where machine learning techniques are used to find the best algorithm
in a space of programs delineated by the human programmer.

\subsection{Our Contribution}

The vast majority of the algorithms described above rely on
machine learning techniques being applied to a set of easily
computable problem features.
This often provides only minimal insight into why a certain technique
works better in a certain context.

To complement previous efforts, this work focuses solely on
analytical insights and expected runtime.
As a starting point we focus
on BFS and DFS expected runtime based on analytically
tractable problem features.
While less immediately applicable, we hope that these kinds
of analyses will ultimately prove valuable in the construction
of flexible search algorithms that make use of a wide range of
problem features.

\section{Preliminaries}
\label{sec:preliminaries}

This section provides various background on material
that will be important for the development of the rest
of the paper.

\paragraph{Graphs and Trees}

A \emph{(directed) graph} is a set $V$ of nodes together
with a set $E$ of \emph{edges}, where
$E\subseteq\{(v_1,v_2):v_1,v_2\in V, v_1\not= v_2\}$.
Throughout we always assume that graphs are directed,
and that edges are represented by ordered pairs $(v_1,v_2)$.
There is a \emph{path from $v_1$ to $v_3$} if there
either is an edge from $v_1$ to $v_3$, or if there is a
node $v_2$ such that there is a path from $v_1$ to $v_2$
and a path from $v_2$ to $v_3$.
When there is a path from $v_1$ to $v_2$, we also say
that $v_1$ and $v_2$ are \emph{connected}, and that
$v_2$ is a \emph{descendant of $v_1$}.
The \emph{length} of a path is the number of edges it contains,
and the \emph{distance} between two nodes is the length of the shortest
path between them (if one exists).
An \emph{undirected graph} is a directed graph where
$(v_2,v_1)$ is an edge whenever $(v_1,v_2)$ is, for any $v_1,v_2\in V$.

A \emph{rooted tree} is a graph with a \emph{root $v_0$}, and
where for every node $v$, there is exactly one path from $v_0$ to $v$.
The \emph{level} of a node $v$ is the distance from the root $v_0$ to $v$.
The \emph{depth $d$} is the length of a
longest path starting from $v_0$.
If every node on level less than $D\in\SetN$ has exactly $b$ children,
and nodes on level $D$ are \emph{leafs} (have no children), then the tree
is \emph{complete with branching factor $b$ and depth $D$}.
Such a tree will have $b^D$ leaves and $(b^{D+1}-1)/(b-1)$ nodes.
In particular, \emph{complete binary trees} (with branching factor 2)
have $2^D$ leaves and $2^{D+1}-1$ nodes.


\subsection{Search Problems}
\label{sec:graph-problems}

A common feature of many search problems is that there are a set of
operations for cheaply modifying a proposed solution into similar
proposed solutions. This makes it natural to view the problem as a
\emph{graph search problem}, where proposed solutions are \emph{states}
or \emph{nodes}.
The modification operations induce \emph{directed edges}.
Sometimes the goal is to find a path to a solution state; sometimes the
solution state itself suffices.
Our results apply to any search problem that fits into
this abstract framework.

Most practical search problems fit into either of the following
two kinds of graph search problems.

\begin{definition}[Constructive graph search problem]
  A \emph{constructive graph search problem} consists of a
  \emph{state space $S$}, a
  starting state $s_0\in S$, and the following efficiently computable
  functions:
  \begin{enumerate}
    \setlength{\parskip}{0pt}
  \item Neighbourhood $N: S \to 2^S$
  \item Goal check $C: S \to \{0,1\}$
  \item Edge cost: $\mathit{EC}:(S \times S) \to \SetR^+$\label{it:edge-cost}
  \end{enumerate}
  A \emph{constructive solution} is a path $s_0,\dots,s_n$ from the starting
  state $s_0$ to a goal state $s_n$ with $C(s_n)=1$.
  The \emph{solution quality} of the path $s_0,\dots,s_n$ is
  $\sum_{i=0}^{n-1}\mathit{EC}(s_i, s_{i+1})$.
  Sometimes a heuristic $g: S \to \SetR^+$ is available to guide
  the search, though we only consider this situation briefly in
  \cref{sec:heuristic}.
\end{definition}

For instance, planning problems are naturally formalised as constructive
graph search problems.
A solution is a \emph{plan} (a sequence of actions) that
transforms the starting state into a goal state.
The neighbourhood function gives a list of states reachable by a single
action from a state.
The goal check indicates whether a state is a goal,
and the edge cost indicates how costly it is to use a certain action
(how it affects the solution quality).
In this work we will assume that the edge cost is 1 for all edges.

A heuristic may give an estimate of how close the given state is to a goal
state (in terms of edge cost).
In this paper, we disregard the additional complexities arising
from the use of heuristic functions
(for details, see \citet{Pearl1984,Russell2010,Edelkamp2012}).

A second kind of graph search problems are
problems where only the final solution matters, and not the path of how
to get there.
These problems are sometimes called local search problems:

\begin{definition}[Local graph search problem]
A \emph{local graph search problem} consists of a \emph{state space $S$}
together with the following efficiently computable functions:
\begin{enumerate}
\setlength{\parskip}{0pt}
\item Neighbourhood $N: S \to 2^S$
\item Constraint $C: S \to \{0, 1\}$
\item Objective function $Q: S\to \SetR$
\end{enumerate}
A \emph{local solution} is a state $s\in S$, $C(s)=1$, and its
\emph{solution quality} is $Q(s)$.
\end{definition}
In local graph search problems, the goal is to find a $v\in S$ that satisfies the
constraints $C$ and achieves as high objective value as possible.
The search for an optimal circuit layout is one example of a problem
that naturally formalises as a local graph search problems.
Neighbours are reached by modifying the current layout (changing one
connection), and
the objective function incorporates the component cost and
the energy efficiency of the layout.
The constraint disqualifies circuits that fail the specifications.

Any constructive search problem $G_1=\langle S_1,N_1,C_1, \mathit{EC}\rangle$
may be formulated as local search
problem $G_2 = \langle S_2, N_2, C_2, Q\rangle$, by letting
\begin{itemize}
\setlength{\parskip}{0pt}
\item the state space $S_2$ be the set of paths in the original problem $G_1$,
\item the objective $Q$ be to minimise the sum of the path cost,
\item the constraint $C_2$ check whether the last node of the path is a goal node, and
\item the neighbourhood function $N_2$ extend or contract a path by
  adding or removing a final node according to $N_1$
  (better choices of $N_2$ may be available).
\end{itemize}
For example, the travelling salesman problem can be viewed as either a
\emph{constructive} problem where a path is built step-by-step, or
as a \emph{local} problem where a full path is modified by swapping edges, and the objective function equals the summed edge cost.
Some potentially useful structure may be lost in the conversion
from a constructive to a local problem, however.

Although mixtures of local and constructive search problems are
possible (e.g., combining an objective function with
a constructive solution and edge cost),
most practical graph search problems naturally
formalise as either a constructive or a local graph search problem.
%
%
%
In this paper, we will focus solely on problems with a
binary distinction between goal and non-goal.
Both constructive and local search problems can get binary
goal predicates by choosing a threshold for maximum
total edge cost or minimum solution quality.

\subsection{Basic Search Algorithms}
\label{sec:graph-algorithms}


A \emph{search algorithm} is an algorithm that returns a solution
(a state or a path) to a graph search problem, given oracle access to
the functions $N$ and $C$, and possibly either to $\mathit{EC}$ and $h$,
or to $Q$
(depending on the type of the search problem).

\emph{Uninformed search} refers to the case where
neither a heuristic function nor an objective function is used
to guide the search.
The two standard methods for exploring a graph in this case
are BFS and DFS.
BFS searches a successively growing neighbourhood around
the the start node, while DFS follows a single path as long as
possible, and \emph{backtracks} when stuck.
Depending on the positions of the goals in the graph, BFS and DFS may
have substantially different performance.
The search orders are illustrated in \cref{fig:bfs-vs-dfs}
(and \cref{fig:bfs-vs-dfs-graph} on page \pageref{fig:bfs-vs-dfs-graph} below).

\begin{algorithm}
  \begin{algorithmic}
    \State Q $\gets$ emtpyQueue
    \State Discovered $\gets$ emptySet
    \State Q.add(start-node)
    \State Discovered.add(start-node)
    \While{Q not empty}
    \State $u\gets $Q.pop()
    \If{$C(u)$}
    \Return $u$\Comment{$u$ is goal}
    \EndIf
    \For{$v$ in $N(u)$}
    \If{tree search \textbf{or} not $v\in$ Discovered}
    \State Q.add($v$)
    \If{graph search}
    \State Discovered.add($v$)
    \EndIf
    \EndIf
    \EndFor
    \EndWhile
  \end{algorithmic}
  \caption{Pseudo-code for BFS (tree search or graph search)}
  \label{alg:bfs}
\end{algorithm}


\begin{algorithm}
  \begin{algorithmic}
    \State path $\gets$ empty list
    \State \Call{DFS-tree-rec}{$N$, $C$, start node, path, radius}
    \State
    \Function{DFS-tree-rec}{$N, C, u$, path, radius}
    \State path.append($u$)
    \If{$C(u)$}
    \Return $u$ \Comment{$u$ is goal}
    \EndIf
    \If{length(path) $<$ radius}
    \For{$v$ in $N(u)\setminus$path}
    \State\Call{DFS-tree-rec}{$N, C, v$, path, radius}
    \EndFor
    \EndIf
    \EndFunction
  \end{algorithmic}
  \caption{Depth-bounded DFS tree search}
  \label{alg:dfs-tree}
\end{algorithm}

\begin{algorithm}
  \begin{algorithmic}
    \State path $\gets$ empty list
    \State visited $\gets$ empty set
    \State \Call{DFS-graph-rec}{$N$, $C$, start node, path, radius, visited}
    \State

    \Function{DFS-graph-rec}{$N, C, u$, path, radius, visited}
    \State visited.add($u$)
    \State path.append($u$)
    \If{$C(u)$}
    \Return $u$ \Comment{$u$ is goal}
    \EndIf
    \If{length(path) $<$ radius}
    \For{$v$ in $N(u)\setminus$visited}
    \State \Call{DFS-graph-rec}{$N, C, v$, path, radius, visited}
    \EndFor
    \EndIf
    \EndFunction
  \end{algorithmic}
  \caption{Depth-bounded DFS graph search}
  \label{alg:dfs-graph}
\end{algorithm}
\todo{Change C to isGoal?}

\begin{figure}
  \centering
  \begin{subfigure}[l]{0.49\textwidth}
    \centering
    \includegraphics{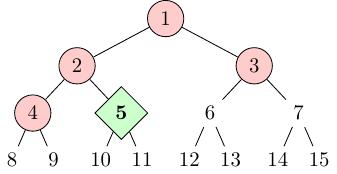}
    \caption{BFS}
  \end{subfigure}
  \begin{subfigure}[r]{0.49\textwidth}
    \centering
    \includegraphics{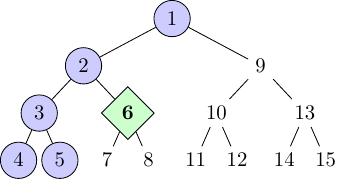}
    \caption{DFS}
  \end{subfigure}
  \caption{%
    The difference between BFS (left) and DFS (right) in
    a complete binary tree where a goal (diamond) is placed in the second
    position on level 2 (the third row).
    The numbers indicate traversal order.
    Circled nodes are explored before the goal is found.
    Note how BFS and DFS explore different parts of the tree.
    In bigger trees, this may lead to substantial differences
    in search performance.
    }
  \label{fig:bfs-vs-dfs}
\end{figure}

\paragraph{Tree search and graph search}
BFS and DFS come in two flavors, depending on whether they keep
track of visited nodes or not.
The \emph{tree search} variants do not keep track of visited nodes,
while the \emph{graph search} variants do.
In trees (where each node can only be reached through one path),
nothing is gained by keeping track of visited nodes.
In contrast,
keeping track of visited nodes can benefit search performance greatly
in multiply connected graphs,
although especially for DFS
the additional memory consumption may sometimes be prohibitive.
\Cref{alg:bfs} describes BFS tree search and graph search.
DFS tree search (\cref{alg:dfs-tree}) is substantially more
memory-efficient than DFS graph search (\cref{alg:dfs-tree})
and BFS: $O(d)$ compared to $O(b^d)$.
However, BFS can be emulated by \emph{iterative deepening DFS} (ID-DFS).
ID-DFS uses the same amount of memory as DFS tree search,
and only has a slightly longer runtime than BFS in most graphs\footnote{
  Assuming exponentially growing neighbourhoods and
  unit edge cost}
\citep[Sec.~3.4.5]{Russell2010}.

For general graphs, we consider DFS with bounded search depth.
Without a bound, a single path may span the entire or a very large
portion of the search space,
giving the search more the characteristics of a random walk than of
search with backtrack.
An unbounded DFS tree search may require as much memory as a BFS search.
This justifies the study of depth-bounded DFS tree search (\cref{alg:dfs-tree}).
A depth-bounded DFS graph search may be analysed with almost the same method,
and is interesting for comparison.
Unfortunately,  depth-bounded DFS graph search is not a \emph{complete}
search method in general graphs, as the search might cut itself off from
regions of the search space. (See \cref{fig:dfs-cut-off} for an example.)
In trees, the search strategies of DFS tree search and DFS graph search
are indistinguishable.


\begin{figure}
  \centering

  \includegraphics{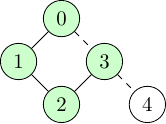}
  \caption{
    Depth-bounded DFS graph search with radius 3 cutting itself off from node 4.
    After node 3 has been visited, the search backtracks to the root node 0.
    While node 4 could originally have been visited via 0--3--4, this
    path is now blocked since node 3 already has been visited.
    DFS tree search does not have this problem.
  }
  \label{fig:dfs-cut-off}
\end{figure}



\subsection{Algorithm performance}

Performance on a single problem may be defined in terms of:
\begin{enumerate}
\setlength{\parskip}{0pt}
\item Solution quality.
\item The number of \emph{explored states} (a state $s$ is
  \emph{explored} if either $N(s)$ or $C(s)$ has been called).
\item The running time of the algorithm.
\item The memory consumption of the algorithm (typically measured by
  the maximum number of states kept in memory).
\end{enumerate}
(Asymptotic) average or worst-case analysis may be used when
measuring performance on a class of problems.

In this work, we will measure performance by the average number of
explored states until a goal is found; that is, item 2
and assuming only the first satisfactory goal matters.
In many cases the number of
explored states is proportional to the actual runtime (item 3),
since state expansion often is the dominant operation
during search.
We therefore permit ourselves to refer to
the number of nodes explored until a first goal is found
as the \emph{runtime} or \emph{search time}.
For example, the runtime of BFS is 5 and the runtime of DFS is
6 in \cref{fig:bfs-vs-dfs}.
If no goal exists,
the search method will explore all nodes before halting. In this case,
we define the runtime as the number of nodes in the search problem plus 1%
\ (i.e., $2^{D+1}$ in the case of a binary tree of depth $D$).\footnote{%
It may have seem more justified to set the non-goal case to
the exact number of nodes instead of adding 1.
However, adding 1 makes most expressions slightly more elegant,
and does not affect the results in any substantial way.
}

\subsection{Probability Theory}

The random variables $X_1,\dots,X_n$ are
\emph{independent and identically distributed (iid)} if
for all $i\in\{1,\dots,n\}$ and any outcome $x$,
$P(X_i\leq x)=P(X_1\leq x)$,
and the probability of any joint outcome $x_1,\dots,x_n$ satisfies
$P(X_1\leq x_1,\dots,X_n\leq x_n)=\prod_{i=1}^n P(X_i\leq x_i)$.

A random variable $X$ is \emph{geometrically distributed $\Geo(p)$} if
$P(X=k) = (1-p)^{k-1}p$ for $k\in\{1,2,\dots\}$.
The interpretation of
$X$ is the number of trials until the first success
when each trial succeeds with iid probability $p$.
Its cumulative distribution function (CDF) is $P(X\leq k) = 1-(1-p)^k$,
and its \emph{average} or \emph{expected value} is $\E[X]=1/p$.
When success is guaranteed to occur within the first $m$
trials, a \emph{truncated geometric distribution} arises.
A random variable $Y$ is
\emph{truncated geometrically distributed $X\sim\truncGeo(p, m)$} if
$Y = (X\mid X\leq m)$ for $X\sim\Geo(p)$, which gives
\begin{align*}
  P(Y=k)
    &= \begin{cases}
         \frac{(1-p)^kp}{1-(1-p)^m} & \text{for }k\in\{1,\dots, m\}\\
         0 & \text{otherwise.}
       \end{cases}\\
 \E[Y]=\tc(p,m)
    &= \E[X\mid X\leq m] = \frac{1-(1-p)^m(pm+1)}{p(1-(1-p)^m)}.
\end{align*}
When $p\gg \frac{1}{m}$, $Y$ is approximately $\Geo(p)$, and $\tc(p,m)\approx\frac{1}{p}$.
When $p\ll \frac{1}{m}$, $Y$ becomes approximately uniform on
$\{1,\dots,m\}$ and $\tc(p,m)\approx\frac{m}{2}$.

A random variable $Z$ is \emph{exponentially distributed $\Exp(\lambda)$}
if $P(Z\leq z) = 1 - e^{-\lambda z}$ for $z\geq 0$.
The expected value of $Z$ is $1/\lambda$, and the
probability density function of $Z$ is $\lambda e^{-\lambda z}$.
An exponential distribution with parameter $\lambda = -\ln(1-p)$
can be viewed as the continuous counterpart of a $\Geo(p)$
distribution.
We will use this approximation in \cref{sec:mgl}.
\begin{lemma}[Exponential approximation]\label{le:exp-geo-approx}
Let $Z\sim\Exp(-\ln(1-p))$ and $X\sim\Geo(p)$.
Then the CDFs for $X$ and $Z$ agree for integers $k$,
$P(Z\leq k)  = P(X \leq k)$.
The expectations of $Z$ and $X$ are also similar in the sense that
$0\leq \E[X]-\E[Z] \leq 1$.
\end{lemma}

\begin{proof}
For $z>0$, $P(Z\leq z) = 1-\exp(z \ln(1-p)) = 1-(1-p)^z$,
and $P(X \leq z) = 1-(1-p)^{\lfloor z \rfloor}$. Thus,
for integers $k>0$, $P(Z \leq k) = P(X\leq k)$ which proves the
first statement.
Further, $1-(1-p)^{\lfloor z\rfloor} \leq 1-(1-p)^{z} < 1-(1-p)^{\lfloor z+1 \rfloor} $,
so $P(X\leq z)\leq P(Z\leq z) < P(X-1 \leq z)$.
Hence $\E[X]\geq \E[Z] > E[X-1]=\E[X]-1$, which proves the
second statement.
\end{proof}

We will occasionally make use of the convention $0\cdot \mathit{undefined}=0$.

Let $\Omega$ be a \emph{sample space}, i.e.\ a set of possible outcomes.
The Law of Total Expectation allows us
to expand expectations by conditioning on disjoint events:

\begin{lemma}\label{le:expand-expectation}
Let $X$ be a random variable and let the sample space
$\Omega=\dot\bigcup_{i\in I} C_i$ be
partitioned by mutually disjoint events $C_i$.
Then $\E[X] = \E[\E[X\mid C_i]]=\sum_{i\in I} P(C_i)E[X\mid C_i]$.
\end{lemma}

This concludes the background section, and the stage is
now set for the analysis proper.

\section{Tree with a Single Goal Level}
\label{sec:sgl}

We start with analysing expected BFS and DFS runtime in trees.
The results apply when the search graph is a tree, and when
tree search versions of BFS and DFS are used
(in which case any graph ``looks like'' a tree, as discussed
in \cref{sec:graph-algorithms}).
This section assumes that all goals are located on a single level
of the tree; i.e., all goals have the same distance from the start
node.
This is usually unrealistic, but makes the analysis easier.
The next section relaxes the assumption of a single goal level.

Our aim throughout is to derive closed-form approximations
for BFS and DFS expected search time.
\cref{fig:bfs-vs-dfs} illustrates
the different search strategies BFS and DFS,
and how they initially focus the search on different areas of the tree: BFS stays close to the root while DFS goes directly to the bottom.
In this section only, the comparison between BFS and DFS
expected search time yields an elegant decision boundary
between which method is better in expectation.

As a concrete example,
consider the problem of solving a Rubik's cube.
\citet{Rokicki2013} did a thorough analysis of this problem,
and found that there is an upper bound to how many moves
it can take to reach the goal, and
that most goals are located around level 17 ($\pm 2$ levels).
If we consider search algorithms that do not remember where they have been,
the search space becomes a complete tree with fixed branching factor 18
(or 13.3 on average, if we cannot immediately return to the preceding state)
\citep{Edelkamp1998}.
What would be the expected BFS and DFS search time for this problem?
Which one would be faster?



\paragraph{The model}
Our single goal level model is defined by the following and
illustrated in \cref{fig:sgl}.
In a binary tree of depth $D$, let solutions be distributed
on a single \emph{goal level $g\in \{0,\dots,D\}$}.
At the goal level, any node is a goal with iid
probability $p_g\in [0,1]$.
We will refer to these kinds of problems as
\emph{(single goal level) complete binary trees with
depth $D$, goal level $g$ and goal probability $p_g$}.

\begin{figure}
  \centering
  \begin{subfigure}{0.49\textwidth}
    \centering
    \includegraphics{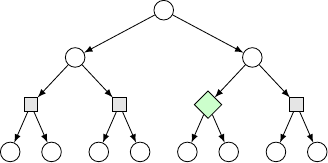}
  \end{subfigure}
  \begin{subfigure}{0.49\textwidth}
    \centering
    \includegraphics{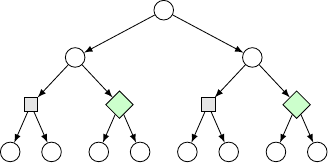}
  \end{subfigure}
  \caption{Two possible outcomes of goal distribution in a single goal
    level problem with max depth $D=3$, goal level
    $g=2$ (boxes) with nodes being goals (diamonds)
    with iid probability $p_g=1/3$.
    Depending on goal locations, BFS and DFS performance will
    differ. We are interested in \emph{expected} performance.
  }
  \label{fig:sgl}
\end{figure}

Note that there may be several or zero goals.
Denote with $\Gamma$ the event that a goal exists,
and $\bar \Gamma$ the event that no goal exists.
It will be useful later to also define
$\Gamma_k$ as the event that a goal exists on level $k$,
and $\bar\Gamma_k$ as its complement.
The probability that a goal exists is $P(\Gamma) =P(\Gamma_g)= 1-(1-p_g)^{2^g}$.
If a goal exists, let $Y\in\{1,\dots,2^g\}$
be the position of the first goal at level $g$.
Conditioned on a goal existing,
$Y$ is a truncated geometric variable $Y\sim\truncGeo(p_g, 2^g)$.
When $p_g\gg 2^{-g}$ the goal position $Y$ is approximately $\Geo(p_g)$,
which makes most expressions slightly more elegant.
This is often a realistic assumption since we usually
expect the problem to have a solution.
If $p\not\gg 2^{-g}$, then the likelihood of no goal is large.
Our analysis does not require that a goal exists.

\paragraph{Runtime estimates}
The following two propositions give runtime estimates
for BFS and DFS by following the counting schemes illustrated
in \cref{fig:sgl-proofs}.
The BFS result is particularly simple.
Throughout the paper, we use $t^{\mathit{alg}}_{{\textit{problem type}}}$ to denote
expected search time for algorithm \emph{alg} on the subscripted problem
type. A tilde on top denotes rough approximation.

\begin{proposition}[BFS runtime Single Goal Level]\label{pr:tbfs-sgl}
  Let the problem be a complete binary tree with depth $D$,
  goal level $g$ and goal probability $p_g$.
  When a goal exists and has position $Y$ on the goal level,
  the BFS search time is
  \begin{align}
    \tbfss(g, p_g, Y)         &= 2^g-1+Y \text{, with expectation}\label{eq:bfs1}\\
    \tbfss(g, p_g\mid \Gamma_g)  &= 2^g-1+\tc(p_g,2^g) \approx 2^g-1+\frac{1}{p_g}\label{eq:bfs-sgl-goal}.
  \end{align}
  In general, when a goal does not necessarily exist,
  the expected BFS search time is
  \begin{equation}\label{eq:bfs-sgl-no-goal}
    \tbfss(g, p_g) = P(\Gamma)\cdot(2^g-1+\tc(p_g,2^g)) + P(\bar\Gamma)\cdot 2^{D+1}
    \approx 2^g-1+\frac{1}{p_g}.
  \end{equation}
  The right hand approximations of \eqref{eq:bfs-sgl-goal}
  and \eqref{eq:bfs-sgl-no-goal} are close when $p_g\gg 2^{-g}$
  and $D\not\gg g$.
\end{proposition}

\begin{figure}
  \centering
  \begin{subfigure}[l]{0.48\textwidth}
    \centering
    \includegraphics{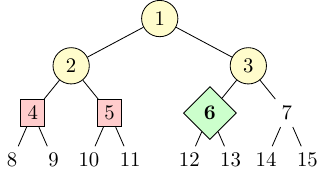}
    \caption{BFS count}
    \label{fig:sgl-proof-bfs}
  \end{subfigure}
  \begin{subfigure}[r]{0.51\textwidth}
    \centering
    \includegraphics{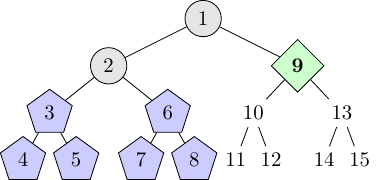}
    \caption{DFS count}
    \label{fig:sgl-proof-dfs}
  \end{subfigure}
  \caption{The counting schemes used in \cref{pr:tdfs-sgl,pr:tbfs-sgl},
    illustrated in a tree of depth 3.
    The BFS count is illustrated with a goal (diamond)
    on the third position on level 2,
    while the DFS count is illustrated with a goal on level 2.
  }
  \label{fig:sgl-proofs}
\end{figure}

\begin{proof}
When a goal exists,
BFS will first explore all of the top of the tree until depth $g-1$:
The $2^{(g-1)+1}=2^g$ nodes that are circles in \cref{fig:sgl-proof-bfs}.
BFS will then search $Y$ nodes on level $g$
(boxes and diamond in \cref{fig:sgl-proof-bfs}).
The total search time is thus $\tbfss(D, g, p_g, Y) = 2^g-1+Y$,
with expected value $2^g-1+\tc(p_g,2^g)$. 

In the general case when a goal does not necessarily exist,
the expected value of the
search time $X$ expands as
\begin{align*}
\E[X] &= P(\Gamma)\cdot \E[X\mid \Gamma] + P(\bar\Gamma)\cdot \E[X\mid \bar\Gamma]\\
      &= P(\Gamma)\cdot\tbfss(D,p,p_g\mid \Gamma_g) + P(\bar\Gamma)\cdot 2^{D+1}\\
      &= P(\Gamma)\cdot(2^g-1+\tc(p_g,2^g)) + P(\bar\Gamma)\cdot 2^{D+1}.
\end{align*}
When $p_g\gg 2^{-g}$, then $Y\approx \Geo(p)$,
and $P(\Gamma)\approx 1$ and $P(\bar\Gamma)\approx 0$.
Further, $2^D\not\gg 2^g$ since $D\not\gg g$, so
the term $P(\bar\Gamma)2^D$ cannot significantly affect the expectation.
This justifies the $(1/p_g-1)2^{D-g+1}$ approximation.
\end{proof}

\label{page:Korf}
\cref{pr:tbfs-sgl} can be compared with the more general result
for IDA* by \citet{Korf2001}.
A memory-efficient tree-search variant of BFS can be implemented as
iterative deepening DFS (ID-DFS).
The runtime of ID-DFS is about twice the runtime of BFS.
\citeauthor{Korf2001}'s bound comes out as
$\tbfss(g)\approx 2^{g+2}$, which corresponds to a doubling of
the \emph{worst case}\footnote{%
To be precise, $\tbfss(g)\approx 2^{g+2}$ is obtained from
\citet[Th.\ 1]{Korf2001} by setting:
The heuristic $h=0$, the number of $i$-level nodes $N_i=2^i$,
the equilibrium distribution $P(x)=1$, the edge cost $=1$,
and the cost bound $c$ equal to our max depth $D$.
Their bound then comes out as 
$\tbfss(g)\approx 2^{g+2}$ after iteration over all levels $\leq g$.}
of $Y=2^g$ in \eqref{eq:bfs1}. 
The doubling is correct since ID-DFS is twice as slow as
BFS in the worst case.

We next turn to analyse DFS in a similar manner.

\begin{proposition}[DFS runtime Single Goal Level]\label{pr:tdfs-sgl}
  Consider a complete binary tree with depth $D$, goal level $g$ and
  goal probability $p_g$.
  When a goal exists and has position $Y$ on the goal level,
  the DFS search time is approximately
  \begin{align}
    \tdfss(D, g, p_g, Y)
    &:= (Y-1)2^{D-g+1}+2\text{, with expectation}\nonumber\\
    \tdfss(D, g, p_g\mid \Gamma_g)
    &:=\left(\tc(p_g,2^g)-1\right)2^{D-g+1}+2 \approx
      \left(\frac{1}{p_g}-1\right)2^{D-g+1}+2.\label{eq:sgl-goal}
  \end{align}
  The expected DFS search time when a goal does not necessarily
  exist is approximately
  \begin{equation}\label{eq:sgl-no-goal}
    \tdfss(D, g, p_g)  :=
    P(\Gamma)((\tc(p_g, 2^g)-1)2^{D-g+1}+2) + P(\bar\Gamma)2^{D+1}
    \!\approx\!\! \left(\frac{1}{p_g}\!-\!1\!\right)\!2^{D-g+1}\!\!.
  \end{equation}
  The right hand approximations in \eqref{eq:sgl-goal}
  and \eqref{eq:sgl-no-goal} are valid when $p_g\gg 2^{-g}$.
\end{proposition}

\begin{proof}
One way to count the nodes explored by DFS when a goal exists
is the following.
To the left of the first goal on level $g$, DFS will explore
$2(Y-1)$ subtrees rooted at level $g+1$
(pentagons in \cref{fig:sgl-proof-dfs}).
These subtrees will have depth
$D-(g+1)$, and contain $2^{D-g}-1$ nodes each.
DFS will also explore
$Y$ nodes on level $g$ and their parents, which amounts to
about $2Y$ nodes (circles in \cref{fig:sgl-proof-dfs}). \todo{Make precise}
Summing the contributions gives the DFS search time approximation
$\tdfss(D,g,p_g,Y)=2(Y-1)\cdot(2^{D-g}-1)+2Y = (Y-1)2^{D-g+1}+2$.

By \cref{le:expand-expectation}, the expected value of the search time $X$
when a goal does not necessarily exist
expands as
\begin{align*}
\E[X] &= P(\Gamma)\cdot\E[X\mid \Gamma] + P(\bar\Gamma)\cdot\E[X\mid \bar\Gamma]\\
      &= P(\Gamma)\cdot\E[\tdfss(D,g,p_g,Y)\mid \Gamma] + P(\bar\Gamma)\cdot 2^{D+1}\\
      &= P(\Gamma)\cdot((\tc(p_g, 2^g)-1)2^{D-g+1}+2) + P(\bar\Gamma)\cdot 2^{D+1}
\end{align*}
where the last step uses that $(Y\mid\Gamma)\sim\truncGeo(p_g,2^g)$.
When $p_g\gg 2^{-g}$, then
$\Gamma\approx 1$, $\bar\Gamma\approx 0$ and $Y\approx \Geo(p_g)$
which justifies the approximation.
\end{proof}


%

\cref{pr:tdfs-sgl,pr:tbfs-sgl} provide expected runtime estimates
as a function of the parameters $D$, $g$, and $p_g$.
\cref{fig:search-time} in \cref{sec:experimental-verification}
plots the runtime estimates as functions of the goal level $g$.
As expected, one observation that can be made from these results is that
BFS benefits
when goals are close to the start node, and DFS benefits
when goals are close to the maximum search depth of the tree.
This can be seen from positive $g$ exponent in \cref{pr:tbfs-sgl}
for BFS,
and the negative $g$ exponent in \cref{pr:tdfs-sgl} for DFS.
The runtimes are also plotted as a function of $g$ in \cref{fig:search-time}.
Although the model is unrealistic, these results provide
important building blocks for the more general models in subsequent
sections.

\paragraph{Decision boundary}
An interesting point to analyse is the crossover
where DFS overtakes BFS in performance.
This crossover occurs where the difference $\tbfss-\tdfss$ between
BFS and DFS runtimes shifts sign.
It turns out that this crossover has an elegant expression:

\begin{proposition}[Decision boundary for single goal level binary tree]
\label{pr:bfs-vs-dfs-sgl}
Let $\gamma_{p_g} = \log_2\left(\tc(p_g, 2^g)-1\right)/2
\approx \log_2\left(\frac{1-p_g}{p_g}\right)/2$.
Given the approximation of DFS runtime of \cref{pr:tdfs-sgl}, BFS wins in
expectation in a complete binary tree with depth $D$, goal level $g$ and
goal probability $p_g$ when
\[g < \frac{D}{2} + \gamma_{p_g}\]
and DFS wins in expectation when
$g > \frac{D}{2} + \gamma_{p_g} + \frac{1}{2}$.
\end{proposition}

The approximation $\gamma_{p_g} \approx \log_2\left(\frac{1-p_g}{p_g}\right)/2$
is valid when $p_g\gg 2^{-g}$.
The proposition holds regardless of this assumption.

\begin{proof} 
When no goal exists, BFS and DFS will perform the same.
When the tree contains at least one goal node,
BFS will find the goal somewhere on its sweep across level $g$,
so the BFS runtime is bounded between $2^{g}\leq \tbfss(g,p_g)\leq 2^{g+1}$.

The upper bound for $\tbfss(g,p_g)$ gives that
$\tbfss(g,p_g) < \tdfs(D,g,p_g)$ when
$2^{g+1} < \left(\tc(p_g, 2^g)-1\right)2^{D-g+1}$.
Taking the binary logarithm of both sides yields
\begin{align*}
g+1 &< \log_2\left(\tc(p_g, 2^g)-1\right) + D -g + 1.\\
\intertext{Collecting the $g$'s on one side and dividing by 2 gives the desired bound}
g &< \frac{\log_2(\tc(p_g, 2^g)-1)}{2} + \frac{D}{2}= \frac{D}{2} + \gamma_{p_g}.
\end{align*}

Similar calculations with the lower bound for $\tbfss(g,p_g)$
gives the condition for $\tdfss(D,g,p_g)<\tbfss(g,p_g)$ when
$g>\frac{D}{2}+\gamma_{p_g}+\frac{1}{2}$.
\end{proof}

The term $\gamma_{p_g}$ is in the range $[-1,1]$ when $p_g\in[0.2,0.75]$,
$g\geq 2$, in which case
\cref{pr:bfs-vs-dfs-sgl} roughly says that BFS wins (in expectation)
when the goal level $g$ is located higher than the middle of the tree.
That the decision boundary is halfway between top and bottom
is somewhat surprising given the different natures of the
explored areas of BFS and DFS.
While BFS exhaustively explores one subtree at the top,
DFS typically exhaustively explores several lower subtrees next
to the bottom (see \cref{fig:sgl-proofs}).
Note that the goal probability $p_g$ needs to be quite large for this
balance to occur.

For smaller, more realistic $p_g$, BFS benefits with the boundary being
shifted $\gamma_{p_g}\approx k$ levels from the middle
for $p_g\approx 2^{-2k}$.
In other words, DFS benefits to a greater degree than BFS
from a high goal probability.
The reason is that when the goal probability is very high,
the best search strategy is to follow an arbitrary path down the
tree.
With high probability the path will hit a goal.
When the goal probability is smaller, substantial backtracking will
be required.
%
\Cref{fig:decision-boundaries} on page \pageref{fig:decision-boundaries}
illustrates the decision boundary as a
function of goal depth and tree depth for a fixed probability $p_g=0.07$,
and shows that \cref{pr:bfs-vs-dfs-sgl} can be used to accurately
predict whether BFS or DFS will be faster.

It is straightforward to generalise the calculations to
arbitrary branching factor $b$ by substituting the 2 in the base
of $\tbfss$ and $\tdfss$ for $b$.
In \cref{pr:bfs-vs-dfs-sgl},
the change only affects the base of the logarithm in $\gamma_{p_g}$:

\begin{corollary}[Decision boundary general]\label{co:bfs-vs-dfs-sgl}
Given the above approximations to BFS and DFS runtime, BFS wins in
expectation
in a complete tree with integer branching factor $b\geq 2$,
depth $D$, goal level $g$, and goal probability $p_g$ when
$g < \frac{D}{2} + \gamma_{b,p_g}$,
and DFS wins in expectation when
$g > \frac{D}{2} + \gamma_{b,p_g} + \frac{1}{2}$,
where $\gamma_{b,p_g} = \log_b\left(\tc(p_g, b^g)-1\right)/2 \approx  \log_b(\frac{1-p_g}{p_g})/2$.
\end{corollary}

The approximation $\gamma_{b,p_g} \approx \log_2\left(\frac{1-p_g}{p_g}\right)/2$
is valid when $p_g\gg b^{-g}$, but the result does not otherwise
depend on this assumption.
The clean results obtained in this section are encouraging.
The next section relaxes the arguably unrealistic assumption of
a single goal level.

\section{Tree with Multiple Goal Levels}
\label{sec:mgl}

We now generalise the model developed in the previous section to
problems that can have goals on any number of levels.
Approximate expected runtime results are obtained for both
BFS and DFS.
The BFS analysis is a straightforward generalisation
of the techniques in the previous section.
The DFS analysis requires a bit more work and an
additional approximation of the distribution of
the position of the first goal on a level.

The model is the following.
For each level
$k\in\{0,\dots,D\}$, let $p_k$ be the associated \emph{goal probability}.
Not every $p_k$ should be equal to 0.
Nodes are goals or not independently of each other.
Nodes on level $k$ have probability $p_k$ of being goals.
Let $Y_k$ be the position of the first goal on level $k$ if such
a goal exists.
We will refer to these kinds of problems as
\emph{(multi goal level) complete binary trees with depth $D$ and
  goal probabilities $\p=[p_0,\dots,p_D]$.}
An example is depicted in \cref{fig:mgl}.

\begin{figure}
  \centering
  \includegraphics{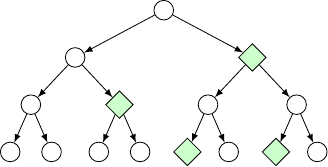}
  \caption{Tree with Multiple Goal Levels model of depth $D=3$
  and goal generated from the goal probability vector $\p=[0,\frac{1}{3},\frac{1}{3},\frac{1}{3}]$}
  \label{fig:mgl}
\end{figure}

Permitting goals on any level with different probability for each level
makes the model significantly more realistic,
as in most cases goals are not located on a single
goal level.
A major open question that remains is how to estimate the goal
probability vector in practice. We discuss this further in
\cref{sec:conclusions}.

\paragraph{Notation}
Let $\Gamma_i$ be the event that level $i$ has a goal,
with $P(\Gamma_i) = 1-(1-p_i)^{2^i}$.
As before, let $\Gamma=\bigcup_i \Gamma_i$ be the event that a
goal exists, and let $\bar \Gamma$ and $\bar\Gamma_i$ be their complements.

\subsection{DFS Analysis}
\label{sec:mgl-dfs}

To find an approximation of goal DFS expected runtime in trees with multiple
goal levels, we
approximate the geometric distribution used in \cref{pr:tdfs-sgl}
with an exponential distribution (its continuous approximation by \cref{le:exp-geo-approx}).

\begin{proposition}[DFS runtime for multiple goal levels]\label{pr:dfs-mgl}
Consider a complete binary tree of depth $D$ with goal probabilities
$\p=[p_0,\dots,p_D]\in [0,1)^{D+1}$.
If for all $k$, $p_k\ll 1$ and $Y_k\sim\Geo(p_k)$,
then the expected number of nodes DFS will search is approximately
\[\tdfsm(D, \p)
  := P(\Gamma)/\sum_{k=0}^D\ln(1-p_k)^{-1}2^{-(D-k+1)} + P(\bar\Gamma)2^{D+1}.\]
\end{proposition}

The assumption $Y_k\sim\Geo(p_k)$ is approximately true
when $p_k\gg 2^{-k}$.
If some level $k$ has a smaller $p_k$,
then the probability that the search encounters a goal
at this level is small.
Thus, we expect the result to be approximately true even if
$Y_k\sim\Geo(p_k)$ only for some levels.
Empirical results in \cref{sec:experimental-verification} verify the validity of
the approximations.

The proof constructs for each level $k$ an exponential
random variable $X_k$ that approximates
the search time before a goal is found on level $k$ (disregarding goals on
other levels).
The minimum of all $X_k$ then becomes an approximation of the search
time to find a goal on some level.
The approximations use exponential variables for easy minimisation.

\begin{proof}[Proof of \cref{pr:dfs-mgl}]
The second term $P(\bar\Gamma)2^{D+1}$ covers the case of
no goal being present, and follows immediately from
\cref{le:expand-expectation} and the search time being $2^{D+1}$
when no goal exists.

For the more interesting case of a goal existing,
the proof uses two approximations.
First approximate the position of the first goal on level $k$ with
$Y_k\approx\Exp(\lambda_k)$, where $\lambda_k=-\ln(1-p_k)$.
The approximation is justifiable by \cref{le:exp-geo-approx},
since we assumed $Y_k\sim\Geo(p_k)$.

Second, disregarding goals on levels other than $k$,
the total number of nodes that DFS needs to search
before reaching a goal on level $k$ is approximately
$X_k\sim \Exp(\lambda_k2^{-(D-k+1)})$.
This follows from an approximation of \cref{pr:tdfs-sgl}:
The number of nodes DFS needs to search to find a goal on level $k$
is
\[\tdfss(D, k, p_k, Y_k) = (Y_k-1)2^{D-k+1}+2 \approx Y_k\cdot2^{D-k+1}.\]
(This is a reasonable estimate if $Y_k$ is large, which is likely
given that $p_k\ll 1$ by assumption.)
So $X_k$ is approximately a multiple $2^{D-k+1}$ of $Y_k$.
For any exponential random variable $Z$ with parameter $\lambda$,
the scaled variable $m\cdot Z$ is $\Exp(\lambda/m)$.
This completes the justification of the second approximation.

The result now follows by a standard minimisation of exponential variables.
Since $X_k$ is the number of nodes searched before finding a goal
on level $k$,
the number of nodes searched before finding a goal on any level
is $X=\min_kX_k$.
The CDF for $X$ is approximately
\begin{align*}
  P(X\leq y) &= 1-\prod_{k=0}^D P(X_k>y)\\
             &= 1- \prod_{k=0}^D\exp(-\lambda_k2^{-(D-k+1)}y)\\
             &= 1 -\exp(-y\sum_{k=0}^D\lambda_k2^{-(D-k+1)}).
\end{align*}
(The minimum of exponential variables $Z_k\sim\Exp(\xi_k)$
is again an exponential variable $\Exp(\sum\xi_k)$.)

Thus the search time when a goal exists is
$X\sim\Exp(\sum_{k=0}^D\lambda_k2^{-(D-k+1)}))$,
so the expected search time is
$1/\sum_{k=0}^D\lambda_k2^{-(D-k+1)})$.
This completes the analysis of the case where a goal exists.
Finally multiplying with the probability $P(\Gamma)$ that a goal exists
justifies the first term in the approximation
(compare \cref{le:expand-expectation}).
\end{proof}

In the special case of a single goal level $j$ with $p_j\gg 2^{-j}$,
the result
of \cref{pr:dfs-mgl} is similar to the approximation in \cref{pr:tdfs-sgl}.
When $\p$ only has a single element $p_j\not=0$ and $p_j\gg 2^{-j}$,
the expression $\tdfsm$ simplifies to
\[
  \tdfsm(D,\p)
  = P(\Gamma)\frac{1}{\lambda_j}2^{D-j+1} + P(\bar\Gamma)2^{D+1}
  \approx \frac{1}{\lambda_j}2^{D-j+1}
  = -\frac{1}{\ln(1-p_j)}2^{D-j+1}.
\]
For $p_j$ not close to 1, the factor $-1/\ln(1-p_j)$ is approximately
the same as the corresponding factor $1/p_j-1$ in \cref{pr:tdfs-sgl}
(the \emph{Laurent expansion} is $-1/\ln(1-p_j)=1/p_j-1/2+O(p_j)$).

The DFS runtime result can be adapted to the case where at least
one goal must be present.
Simply replace $P(\Gamma)$ with 1, and remove the second term
$P(\bar\Gamma)2^{D+1}$.

\subsection{BFS Analysis}
\label{sec:mgl-bfs}

The corresponding expected search time $\tbfsm(D,\p)$ for BFS
requires less insight and
can be calculated exactly by conditioning on which level the first goal is.
The resulting formula is less elegant, however.
The same technique cannot be used for DFS, since DFS does
not exhaust levels one by one.

To develop the reduction to the single goal level case, some
extra notation needs to be introduced.
Let $F_k=\Gamma_k\cap(\bigcap_{i=0}^{k-1}\bar\Gamma_i)$ be the event
that level $k$ has the \emph{first} goal.
The probability that level $k$ has the first goal is
$P(F_k)= P(\Gamma_k)\prod_{j=0}^{k-1}P(\bar\Gamma_j)$.
The expected BFS search time gets a more uniform expression by
the introduction of an extra \emph{hypothetical level $D+1$}
where all nodes are goals.
That is, regardless of the goal probabilities of the
problem, we assume that level $D+1$ has goal probability $p_{D+1}=1$ and
$P(F_{D+1})=P(\bar\Gamma) = 1-\sum_{k=0}^DP(F_k)$.

\begin{proposition}[BFS runtime for multiple goal levels]\label{pr:bfs-mgl}
The expected number of nodes $\tbfsm(p)$ that BFS needs to search
to find a goal
in a complete binary tree of depth $D$ with goal probabilities
$\p=[p_0,\dots,p_D]$, is
\[
\tbfsm(\p)
   =\sum_{k=0}^{D+1} P(F_k)\tbfss(k,p_k\mid \Gamma_k)
   \approx \sum_{k=0}^{D+1}P(F_k)\left(2^{k}+\frac{1}{p_k}\right)
\]
\end{proposition}

For $p_k=0$, the expression $\tbfsc(k,p_k)$ and $1/p_k$ will be undefined,
but this only occurs when $P(F_k)$ is also 0.
The BFS runtime estimate can easily be modified to the situation where at
least one goal must exist.
Simply drop the $(D+1)$st term in the sum, and
renormalise the probabilities $P(F_0),\dots,P(F_D)$.

\begin{proof}
To BFS, the event $F_k$ that level $k$ has a goal
is equivalent to the single goal level model of \cref{sec:sgl}.
Let $X$ be BFS search time, and let $(X\mid F_k)$ be the number of nodes
that BFS needs to search when $k$ is the first level with a goal.
Then $(X \mid F_k) = \tbfss(k, p_k, X-(2^{k}-1)\mid\Gamma_k)$,
and $\E[X \mid F_k] = \tbfss(k, p_k\mid \Gamma_k)$.
The result follows by expanding $\E[X]$ over $F_0,\dots,F_{D+1}$
as in \cref{le:expand-expectation}.
\end{proof}

The approximation $\sum_{k=0}^{D+1}P(F_k)\left(2^{k}+\frac{1}{p_k}\right)$
tends to be within a factor 2 of the correct
expression,\footnote{
  Assume $p_k$ approaches 0 for some $k$.
  The difference between $2^k+1/p_k$ and $\tbfss(k,p_k\mid \Gamma_k) = 2^k + \tc(p_k, 2^k)$
  is $1/p_k-\tc(p_k, 2^k) \leq 1/p_k$.
  This difference is multiplied with the probability
  $P(F_k) \leq P(\Gamma_k) = 1-q_k^{2^k}$ where $q_k=1-p_k$.
  Multiplying the probability and the difference  gives $(1-q_k^{2^k})/p_k = (1-q_k^{2^k})/(1-q_k)
  = \sum_{i=0}^{2^k-1}q_k^i\to 2^k$ as $p_k\to 0$ and $q_k\to 1$.
  Thus, the overestimation with $2^k + 1/p_k$ instead of $2^k + \tc(p_k, 2^k)$
  will not exceed a factor 2.
} even when $p_k< 2^{-k}$ for some or all $p_k\in \p$.
The reason is that the corresponding $P(F_k)$'s are small when
the geometric approximation is inaccurate.

\paragraph{Discussion}
\cref{pr:dfs-mgl,pr:bfs-mgl} provide closed-form approximations
for expected runtime of DFS and BFS in graphs with goals
on any number of levels and with essentially any combination
of goal probabilities.
Given knowledge of the goal probabilities, expected BFS and
DFS search time can easily be computed.
Such knowledge is useful when estimating the amount of resources that
will be required to solve a problem, and when deciding whether the problem
is approachable at all.
Expected runtime is often more relevant than worst case runtime,
as most realistic problems may be significantly easier than the
worst ones.

We have not managed to derive a similarly elegant
closed-form decision boundary
as for the single goal level case (\cref{pr:bfs-vs-dfs-sgl}).
However, a simple computer program can still easily compare the
runtime estimates of \cref{pr:dfs-mgl,pr:bfs-mgl} for
a given goal probability vector.
The comparison can be used
to predict the  BFS vs.\ DFS winner.
The decision boundary from this prediction
is plotted for a concrete set of
goal probability vectors in \cref{fig:decision-boundaries}.

The open question of estimating the goal probability vector
is discussed further in \cref{sec:conclusions}.
Both \cref{pr:dfs-mgl,pr:bfs-mgl} naturally generalise
to arbitrary branching factor $b$.


\section{Graph Search}
\label{sec:collide}
\label{sec:cb}


In this section, we explore general graphs.
In addition to analysing the performance of graph search BFS and DFS
(that do remember visited nodes)
we also analyse the performance of tree search DFS in general graphs.
Graph search can significantly improve performance, but in return
requires more memory.
For DFS, the difference is exponential; for BFS only minor.
\cref{fig:bfs-vs-dfs-graph} gives an idea of BFS and
DFS graph search behaviour.

\paragraph{The model}
General, non-tree graphs exhibit significantly more variability
than trees. Graphs vary along dimensions such as
connectivity and path-redundancy, as well as average number
of neighbours.
We capture this variability in what we call a \emph{length-to-depth counter $L$}:

\begin{figure}
  \centering
  \begin{subfigure}[l]{0.49\textwidth}
    \includegraphics{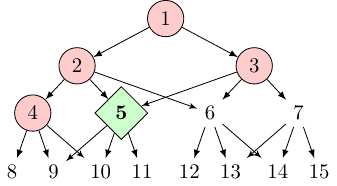}
    \caption{BFS}
  \end{subfigure}
  \begin{subfigure}[r]{0.49\textwidth}
    \includegraphics{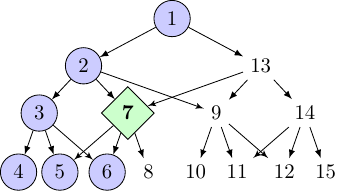}
    \caption{DFS}
  \end{subfigure}
  \caption{The difference between BFS and DFS in
  multiply connected graphs. Note how DFS is additionally concentrated
  at the bottom of the graph compared to \cref{fig:bfs-vs-dfs}.}
  \label{fig:bfs-vs-dfs-graph}
\end{figure}

\begin{definition}[Distance, level, and length-to-depth counter]\label{def:level}
  Let the \emph{distance} $\dist(u, v)$ be the shortest path between $u$ and $v$.
  Let the \emph{level} of a node $v$, $\lvl(v) =\dist(v_0, v)$,
  be the distance from the start node to $v$.
  Let $D=\max_v\lvl(v)$ be the \emph{maximum depth},
  and let $D'$ be the radius of search for DFS.
  Let $\delta_n$ be the first node to which DFS has
  travelled $n$ steps, $0\leq n\leq D'$.

  The \emph{level-to-depth counter $L$} plays a central role in the analysis.
  For a given search problem, let
  \[
    L(n,d)=\E\big[|\{v:\lvl(v)=d, \dist(\delta_n,v)<D'-n\}|\big]
  \]
  be the expected number of nodes on
  level $d$ reachable from $\delta_n$ within the remaining path length $D'-n$.
  Let $\bar L(n, d)$ be the same quantity, but with nodes counted with
  repetition if they can be reached through multiple paths.
\end{definition}

For example, if $D'=2$ then $L(1, 2)$ is the expected number of neighbours
on level 2 after having taken the first search step.
The length-to-depth counter plays the role of a \emph{sufficient statistic}
for search time for graphs.
Although many different graphs have identical length-to-depth counters,
our results below imply that any two graphs with identical length-to-depth counters
will have the same expected search time.
In many graphs, the length-to-depth counter can be connected to
the branching factor (\cref{sec:graph-param}).
As in the previous section, we assume that goals are
distributed by level in an iid manner
according to a goal probability vector $\p$.
We will also assume that the probability of DFS finding a goal before
finding $\delta_D$ is negligible.
We will refer to these kinds of problems as
\emph{search problems with depth $D$, goal probabilities $\p$ and
level-to-depth counter $L$}.
The rest of this section justifies the following proposition.

\begin{proposition}\label{pr:main-cb}
The DFS and BFS runtime of a search problem can be roughly
estimated from the level-to-depth counters $L$ and $\bar L$,
the depth $D$, and the goal probabilities $\p=[p_0,\dots,p_D]$ when
the probability of finding a goal before $\delta_D$ is negligible.\footnote{
  A more careful analysis could relax the assumption of negligible
  probability of finding a goal before $\delta_D$ by
  combining the depth distribution
  $P_n(d)$ defined in \cref{sec:length-to-depth} below with the
  goal probabilities $\p$ and the likelihood of
  an early backtrack.
  These parameters could be used to estimate the probability
  of a goal being found before $\delta_D$, as well as how fast this
  goal would likely be found.
}
\end{proposition}

The assumption of DFS having a negligible probability of finding a
goal before $\delta_D$ is satisfied in problems where
\begin{itemize}
\item nodes typically have several neighbours, so that premature
  backtracking before the full radius is reached usually is not necessary, and
\item no level $k$ has goal probability $p_k$ close to 1.
\end{itemize}
These assumptions are satisfied in a wide range of practical
problems,
including most of the instances investigated in
\cref{sec:experimental-verification}.


\subsection{DFS Analysis}
\label{sec:dfs-cb}

We analyse both DFS tree search and DFS graph search
(\cref{alg:dfs-tree,alg:dfs-graph} on \cpageref{alg:dfs-tree,alg:dfs-graph} above).
Although the analysis in \cref{sec:sgl,sec:mgl} can be used to analyse
DFS tree search in graphs, such an analysis would require an
interpretation of level as path length (as interpreted in \cref{alg:dfs-tree})
rather than shortest distance.
The analysis performed in this section compares nicely with the
corresponding BFS analysis.

\paragraph{Sets of nodes}

Recall that $\delta_n$ is the first node to which DFS has travelled $n$ steps,
and that $D'$ is the radius of search for DFS.
Unless DFS has been forced to backtrack, $\delta_n$ will be the $n$th
node expanded.
We will assume that $\delta_{D'}$ is reached in roughly $D'$ steps.
The nodes $\delta_0,\dots,\delta_{D'}$ play a central role in the
analysis, since the descendants of $\delta_{n+1}$ will be
explored before the descendants of $\delta_n$
(possibly excluding the $\delta_{n+1}$ descendants).
We say that \emph{DFS explores from $\delta_n$} after DFS has explored all
descendants of $\delta_{n+1}$ and until all descendants of $\delta_n$ have
been explored.
The general idea of the DFS analysis will be to count the number of nodes
under each $\delta_n$, and to compute the probability that any of these
nodes is a goal.

Some notation for this (see \cref{fig:cb-notation} for illustration):

\begin{figure}
  \centering
  \includegraphics{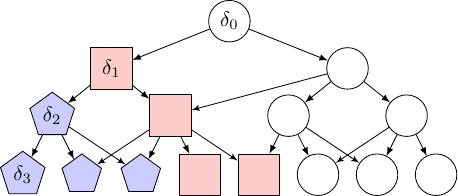}
  \caption{The notation for the DFS graph analysis. Each node $\delta_n$ is
    the first node explored by DFS on level $n$.
    Pentagons denote the $\delta_2$ subgraph $S_2$, and
    boxes the $\delta_1$ explorables $T_1$. The length-to-depth counter $L_{1,3}=5$,
    since 5 nodes on level 3 are reachable from $\delta_1$, and $A_{1,3}=2$
    since 2 nodes on level 3 were not explored before $\delta_1$.
  }
  \label{fig:cb-notation}
\end{figure}

\begin{itemize}
\item Let the \emph{$\delta_n$-subgraph}
  $S_n=\{v:v\in\descendants(\delta_n)\}$ be the set of nodes
  reachable from $\delta_n$,
  and let $\bar S_n = \{v:v\in\overline{\descendants}(\delta_n)\}$
  be the multiset of nodes reachable from $\delta_n$ including repetitions.
  Their expected cardinalities are
  $|S_n|=\sum_{i=0}^DL(n,i)$ and
  $|\bar S_n|=\sum_{i=0}^D\bar L(n,i)$,
  $0\leq n\leq D'$.
  Let $S_{D'+1}=\bar S_{D'+1} = \emptyset$ and let
  $S_{-1}=S_0$ and $\bar S_{-1} = \bar S_0$.
\item Let the \emph{$\delta_n$-explorables} $T_n=S_n\setminus S_{n+1}$
  be the nodes explored from $\delta_n$.
\item Let the \emph{number of level-d $\delta_n$-explorables}
  $A_{n,d}=\max\{0, L(n,d)-L(n+1,d)\}$ be the expected number of level
  $d$ descendants of $\delta_n$ that are not descendants of $\delta_{n+1}$
  for $0\leq d\leq D$ and $0\leq n\leq D'$.
  The relation between $T_n$ and $A_{n,d}$ is the following:
  $|T_n|=\sum_{i=n}^{D'}A_{n,i}$.
\end{itemize}
Let $q_k=1-p_k$ for $0\leq k\leq D$.

\paragraph{DFS search time}
The following lemma establishes the probabilities of finding a goal
under a given $\delta_n$, and is central to \cref{pr:dfs-cb} of
DFS search time.

\begin{lemma}[DFS goal probabilities]\label{le:tn-goal}
  Consider a search problem with depth $D$, goal probabilities $\p$, and
  length-to-depth counter $L$.
  The probability that the $\delta_n$-explorables $T_n$ contains a goal is
  approximately $\tau_n:=1-\prod_{k=0}^Dq_k^{A_{n,k}}$,
  and the probability that $T_n$ contains the first goal is approximately
  $\psi_n:=\tau_n \prod_{i=n+1}^{D'}(1-\tau_i)$.
\end{lemma}
\begin{proof}
  $\tau_n$ is 1 minus the
  probability of \emph{not} hitting a goal at any level $d$, $0\leq d\leq D$,
  since at each level $d$,
  an expected $A_{n,d}$ nodes are visited when exploring from $\delta_n$.
\end{proof}
The probability is not exact, since we disregard the few nodes
explored before $\delta_n$. This slightly affects $A_{n,d}$.

\begin{proposition}[DFS graph search runtime in general graphs]\label{pr:dfs-cb}
  Let $\psi_n$ be the probability of $T_n$ containing the first goal.
  Then the expected DFS search time $t^{\mathrm{DFS}}_{\mathrm{CB}}(D;,\p,L)$
  in a search problem with depth $D$, goal probabilities $\p$,
  and length-to-depth counter $L$ is bounded by
  \begin{equation*}
    \tdfscl(D',\p,L):=\sum_{n=-1}^{D'} |S_{n+1}|\psi_n \leq t^{\mathrm{DFS}}_{\mathrm{CB}}(D',\p,L)
    \leq \sum_{n=-1}^{D'} |S_n| \psi_n:=\tdfscu(D',\p,L)
  \end{equation*}
  where $\psi_{-1} = \bar\Gamma = 1-\sum_{n=0}^{D'}\psi_n$ is the probability that
  no goal exists.
\end{proposition}

The arithmetic mean
$\tdfsc(D',\p,L):=(\tdfscl(D',\p,L)+\tdfscu(D',\p,L))/2$
between the bounds can be used for a single runtime estimate.

\begin{proof}[Proof of \cref{pr:dfs-cb}]
  Let $X$ be the DFS search time in a search problem with the features described
  above.
  The expectation of $X$ may be decomposed as
  \begin{equation}\label{eq:etdfs-informal}
    \E[X]=P(\bar\Gamma)\E[X\mid \bar\Gamma]+\sum_{n=0}^{D'}P(\text{first goal in }T_n)\cdot \E[ X \mid\text{first goal in }T_n].
  \end{equation}

  The conditional search time ($X\mid $ first goal in $T_n$) is bounded by
  $|S_{n+1}|\leq ( X \mid\text{first goal in }T_n)\leq |S_n|$ for $0\leq n\leq {D'}$,
  since to find a goal DFS will search the entire $\delta_{n+1}$-subgraph
  $S_{n+1}$ before finding it when searching the $\delta_n$-explorables $T_n$,
  but will not need to search more than the $\delta_n$-subgraph $S_n = S_{n+1}\cup T_n$
  (assuming no goal is found `on the way down to' $\delta_n$ (i.e.\ to $T_n$)).
  The same bounds also hold with $S_0$ and $S_{-1}$ when no goal exists
  (recall that $|S_{-1}|:=|S_0|+1$).
  Therefore the conditional expectation satisfies
  \begin{equation}\label{eq:nodes-probed}
    |S_{n+1}|\leq \E[ X \mid\text{first goal in }T_n ]\leq |S_n|
  \end{equation}
  for $-1\leq n\leq {D'}$.
  By \cref{le:tn-goal}, the probability that the first goal is among
  the $\delta_n$-explorables $T_n$ is $\psi_n$,
  and the probability $P(\bar\Gamma)$ that no goal exists is $\psi_{-1}$ by definition.

  Substituting $\psi_n$ and
  \eqref{eq:nodes-probed} into \eqref{eq:etdfs-informal}
  gives the desired bounds for expected DFS search time $\tdfsc({D'},\p,L)=\E[X]$.
\end{proof}

\begin{proposition}[DFS tree search runtime in general graphs]\label{pr:dfs-ts}
The expected DFS search time $t^{\mathrm{DFS}}_{\mathrm{CB}}({D'},\p,L)$ in a search problem with
depth $D$, goal probabilities $\p$, and length-to-depth counters $L$ and $\bar L$ is
bounded by
\begin{equation*}
\sum_{n=-1}^{{D'}} |\bar S_{n+1}|\psi_n \leq t^{\mathrm{DFS}}_{\mathrm{CB}}({D'},\p,L, \bar L)
\leq \sum_{n=-1}^{D'} |\bar S_n| \psi_n
\end{equation*}
where $\psi_{-1} = \bar\Gamma = 1-\sum_{n=0}^{D'}\psi_n$ is the probability that
no goal exists.
\end{proposition}

\begin{proof}
  Identical to \cref{pr:dfs-cb}, except nodes may be revisited so
  $|\bar S|$ replaces $|S|$.
  For the chance of finding a goal, the unique count $A_{n,d}$ is still
  the relevant one, so the same $\psi_n$ probability should still be used.
\end{proof}

To refer to the upper and lower bounds of \cref{pr:dfs-ts}, we will use
the notation
\[\tdfsclb({D'},\p,L,\bar L):=\sum_{n=-1}^{{D'}} |\bar S_{n+1}|\psi_n
  \text{ and }\tdfscub({D'},\p,L,\bar L):=\sum_{n=-1}^{D'} |\bar S_n| \psi_n\]
The extra argument $\bar L$ distinguishes the DFS tree search estimates
from the DFS graph search estimates.
As for DFS graph search,
the arithmetic mean
$\tdfsc(D',\p,L, \bar L):=(\tdfscl(D',\p,L, \bar L)+\tdfscu(D',\p,L, \bar L))/2$
between the bounds can be used for a single runtime estimate.
Both the DFS graph search and DFS tree search
runtime estimates are easily modified to the situation where at
least one goal must exist.
Simply drop the $n=-1$ term in the sums, and
renormalise the probabilities $\psi_0,\dots,\psi_{D'}$.

The informativeness of the bounds of \cref{pr:dfs-cb,pr:dfs-ts} depend on
the dispersion of nodes between the different $T_n$'s.
If most nodes belong to one or a few sets $T_n$, the bounds may be
almost completely uninformative.
This happens in the special case of complete trees with branching factor $b$,
where a fraction $(b-1)/b$ of the nodes will be in $T_0$.
The previous section derives techniques for these cases.
The analysis in \cref{sec:grammar,sec:n-puzzle} below
show that the bounds of \cref{pr:dfs-cb,pr:dfs-ts}
may be relevant in more connected graphs.

\subsection{BFS Analysis}

The analysis of BFS only requires the length-to-depth counter $L(0,\cdot)$
with the first argument set to 0, and follows the same structure
as \cref{sec:mgl-bfs}.
In contrast to the DFS bounds above,
this analysis gives a precise expression for the expected runtime.
The idea is to count the number of nodes in the upper $k$ levels of the
tree (derived from $L(0,0),\dots,L(0,k)$), and to compute the probability that
they contain a goal.
Let the \emph{upper subgraph $U_k=\sum_{i=0}^{k-1}L(0,i)$} be the number of
nodes above level $k$.
When there is only a single goal level,
\cref{pr:tbfs-sgl} naturally generalises to the more general setting
of this section:

\begin{lemma}[BFS runtime in graphs with single goal level]\label{le:bfs-sgl-cb}
For a search problem with depth $D$ and
length-to-depth counter $L$, assume that
the problem has a single goal level $g$ with goal probability
$p_g$, and that $p_j=0$ for $j\not=g$.
When a goal exists and has position $Y$ on the goal level,
the BFS search time is:
\begin{align*}
\tbfsc(g,p_g,L, Y) &= U_g + Y \text{, with expected value}\\
\tbfsc(g,p_g, L\mid \Gamma_g) &= U_g + \tc(p_g, L(0,g))
\end{align*}
\end{lemma}

\begin{proof}
When a goal exists,
BFS will explore all of the top of the tree until depth $g-1$
(that is, $U_g$ nodes)
and $Y$ nodes on level $g$ before finding the first goal.
The expected value of $Y$ is $\tc(p_g, L(0,g))$.
\end{proof}

\cref{le:bfs-sgl-cb} generalises to multiple goal levels
analogously to the generalisation made from single goal level
to multiple goal levels in trees.
First note that
the probability that level $k$ has a goal is
$P(\Gamma_k) = 1 - q_k^{L(0,k)}$,
and the probability that level $k$ has the first goal is
$P(F_k) = P(\Gamma_k) \prod_{i=0}^{k-1}P(\bar\Gamma_i)$.
By the same argument that was used in the proof of \cref{pr:bfs-mgl},
the following proposition holds.

\begin{proposition}[BFS runtime in general graphs]\label{pr:bfs-cb}
The expected number of nodes that BFS needs to search
to find a goal in a search problem with depth $D$, goal probabilities
$\p=[p_0,\dots,p_D]$, $\p\not=\mathbf{0}$, and length-to-depth counter $L$ is
\[
\tbfsc(\p, L)
   =\sum_{k=0}^{D+1} P(F_k)\tbfsc(k,p_k, L\mid \Gamma_k)
\]
where the goal probabilities have been extended with an extra element $p_{D+1}=1$,
and $F_{D+1}=\bar\Gamma$ is the event that no goal exists.
\end{proposition}
For $p_k=0$, $\tbfsc$ will be undefined,
but this only occurs when $P(F_k)$ is also 0.
The runtime estimate is easily modified to the situation where at
least one goal must exist.
Simply drop the $(D+1)$st term in the sum, and
renormalise the probabilities $P(F_0),\dots,P(F_D)$

\paragraph{Discussion}
\cref{pr:dfs-cb,pr:bfs-cb} give (rough) estimates of average
BFS and DFS graph search time given the goal distribution $\p$ and
the structure parameter $L$.
The results apply to a very wide range of situations
(where the assumptions are satisfied and the length-to-depth counter
and the goal probability vector can be inferred).
However, the abstract nature of \cref{pr:dfs-cb,pr:bfs-cb}
makes it hard to directly assess their applicability.
This is partially remedied by the concrete examples
in the next section.

\section{Estimating Graph Parameters}
\label{sec:graph-param}

In this section we show that the length-to-depth counters $L$ and $\bar L$
can be estimated from a local sample in graphs with a sufficiently
uniform structure.
In \cref{sec:n-puzzle} we use the techniques developed here to
obtain estimates of the length-to-depth counters for the N-puzzle
and verify the results empirically.

\subsection{Branching Factors}
Our runtime estimates will be based on average \emph{local}
and \emph{global} branching factors $\bup$, $\bside$, $\bdown$
and $\bgup$, $\bgside$, $\bgdown$.
%
Although we will generally assume that graphs are rather uniform in
their properties, a common situation is that graphs consist of a few
different types of nodes.
For example, in the N-Puzzle described in \cref{sec:n-puzzle},
nodes with the empty tile in a corner, touching the edge, or in the
middle have different number of neighbours.
When averaging, the most relevant average is usually with respect to the
\emph{equilibrium distribution} \citep{Edelkamp1998}.
The equilibrium distribution takes into account how likely each type
of node is to be visited.
For example, nodes with few neighbours may be less often visited
than nodes with many neighbours.
The equilibrium distribution can be empirically estimated, or
computed from the transition probabilities between node types
(see \citet{Edelkamp1998} for details).

In trees, each node only branches downward,
with connections to the level just below.
In graphs, the situation is more complex.
In general, a node may be connected to one or several nodes on the
level above, and to zero or more nodes on the same level
and the level below.
Note however, that nodes can only be connected to nodes on the
same or adjacent levels.
If $v$ and $w$ are connected, $w$ can be at most one additional
step away from the root than $v$.

\begin{definition}[Local branching factors]
  \label{def:loc-branch}
  For a given node $v$, let
  \begin{itemize}
  \item the \emph{(local) upwards branching factor $\bup(v)$} be the number of
    neighbours $w$ of $v$ such that $\lvl(w) = \lvl(v)-1$
  \item the \emph{(local) sidewards branching factor $\bside(v)$} be the number of
    neighbours $w$ of $v$ such that $\lvl(w) = \lvl(v)$
  \item the \emph{(local) downwards branching factor $\bdown(v)$} be the number of
    neighbours $w$ of $v$ such that $\lvl(w) = \lvl(v)+1$.
  \end{itemize}
  The definition is illustrated in \cref{fig:graph-node}
\end{definition}

If a node $v$ is not given as an argument,
then $\bup$, $\bside$, and $\bdown$ refer to
the \emph{average branching factors} with respect to the equilibrium distribution.
We will generally assume that the average local branching factors are
similar on all levels (except, possibly, the lowest).
\begin{figure}
  \centering
  \includegraphics{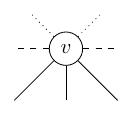}
  \caption{A node $v$ with two connections to nodes on
    the level above ($\bup(v)=2$), two connections to nodes on the same level
    ($\bside(v)=2$) and three connections to nodes on the level below ($\bdown=3$).
    By the definition of level, nodes can only be connected to the same level
    and the levels directly above and below.
  }
  \label{fig:graph-node}
\end{figure}
The local branching factors are local in the sense
that they can easily be determined by looking at a single node.
Alternative, \emph{global}  branching factors can be defined
by considering the ratio between the number of nodes directly
reachable on adjacent levels.

\begin{definition}[Global branching factors]\label{def:glo-branch}
  Let the \emph{global upward, sideward and downward branching factors}
  be defined as
  \begin{align*}
    \bgup[,l](v) &= \frac{|\{w: \lvl(w)=\lvl(v)-l-1, \dist(v,w)=l+1\}|}{|\{w: \lvl(w)=\lvl(v)-l, \dist(v,w)=l\}|}\\
    \bgside[,l](v) &= \frac{|\{w: \lvl(w)=\lvl(v), \dist(v,w)=l+1\}|}{|\{w: \lvl(w)=\lvl(v), \dist(v,w)=l\}|}\\
    \bgdown[,l](v) &= \frac{|\{w: \lvl(w)=\lvl(v)+l+1, \dist(v,w)=l+1\}|}{|\{w: \lvl(w)=\lvl(v)+l, \dist(v,w)=l\}|}
  \end{align*}
  where $v$ is an arbitrary node and $l$ is a natural number small enough
  that the denominator is defined; $\bgdir[,l](v)$ is left undefined when the denominator is 0.
\end{definition}

For example, in the graphs displayed in the \cref{fig:bfs-vs-dfs-graph},
the average local branching factor is approximately 3,
while for the root node $v_0$ the global branching factor is
$\bgdown[,2](v_0) = \text{(nodes on level 3)}/\text{(nodes on level 2)} = 4/2 = 2$ and
$\bgdown[,3](v_0)= \text{(nodes on level 4)}/\text{(nodes on level 3)} = 8/4 = 2$.

The theory will generally rely on a uniformity assumption that the choice
of $v$ and $l$ are not essential for $\bgdir[,l](v)$ as long as they are
chosen within some natural constraints.
This will allow us to drop the arguments $l$ and $v$.
First, $l$ needs to be chosen so that the denominator of
$\bgdir[,l](v)$ is not 0.
For this to be possible, $v$ must be chosen away from the top of the
tree for $\bgup$, and away from the bottom for $\bgdown$.
Finally, we also require $l\geq 2$ since
for $l=1$, the global branching factors equal the local ones.

Note that for trees with constant branching factor $\bgdown = \bdown$
and $\bgup=\bup = \bgside= \bside=1$.
In most graphs and for most directions $\dir\in\{\up,\side,\down\}$,
$\bgdir[,l] \leq \bdir$,
since some paths may ``collide'' and descendants of $v$ share children.%
\footnote{
  We expect the inequality $\bgdir[,l] \leq \bdir$ to hold generally,
  but since the average for $\bdir$ is taken with respect to
  the equilibrium distribution, a proof would be required.
}

%

\paragraph{Discounted branching factors}
Finally, we introduce the notion of a \emph{discounted} branching factor,
to account for the fact that returning to the node just arrived from is
blocked in our search methods.


\begin{definition}[Discounted branching factors]
  For $\dir\in\{\up,\side,\down\}$, let $\bdir'(v) = \bdir(v)-1$
  be the \emph{discounted branching factor in direction $\dir$}.
\end{definition}

The definition is natural, since exactly one neighbour in the direction
the search arrived from will be blocked from return.
When dropping the argument $v$, some care needs to be taken with the
equilibrium distribution.
For example, if half the nodes have a sideward neighbour, and half
the nodes have none, then $\bside=0.5$.
This would give $\bside'=-0.5$ which lacks reasonable interpretation.
Instead, when calculating $\bside'$,
the equilibrium distribution needs to be conditioned
on the fact that the node has been arrived at
from the side.
This implies that the node is the type with one sidewards neighbour.
This sidewards neighbour is now blocked, so $\bside' = 0$.
The subtlety of $\bdir'\not=\bdir-1$ is mainly important 
in graphs with widely varying types of nodes.



We summarise the uniformity assumptions we make for future reference:

\begin{assumption}[Uniformity]\label{as:uniformity}
  We assume that the graph is uniform in the sense that:
  \begin{itemize}
  \item The average branching factors $\bup$, $\bside$, and $\bdown$
    and their discounted counterparts $\bup'$, $\bside'$, and $\bdown'$
    remain the same across levels.
  \item The global branching factors are independent of the choice of $v$
    and $2\leq n\leq \lvl(v)/2$ in \cref{def:glo-branch}.
  \end{itemize}
\end{assumption}

\paragraph{Empirical estimation}

Given uniformity \cref{as:uniformity}, the parameters
$\bup$, $\bside$, $\bdown$, $\bgup$, $\bgside$, and $\bgdown$
can be estimated accurately from a (small) local sample.
When \cref{as:uniformity} is only approximately satisfied,
a larger sample may be required.

\subsection{Length to Depth}
\label{sec:length-to-depth}

\paragraph{Depth transitions}

In graphs, not all new neighbours of a node $v$ are one level below $v$.
Neither is it usually possible to tell which of the new neighbours
are above, beside, or below $v$.
This means that if we follow a path of length $n$ from the root,
we cannot generally tell which level between 0 and $n$ we are at.
However, comparing the (average) number of upwards, sidewards, and downwards
nodes, probabilistic arguments about the depth can still be made.

The direction from which we arrive to the node is blocked from return.
We therefore define the following depth transition probabilities
conditioned on the direction we reach the node from.
\begin{definition}[Depth transition probabilities]
  \label{def:trans-prob}
  Let $b = \bup + \bside + \bdown - 1$.
  Define the following conditional depth transition probabilities $p_\mathrm{arr,dir}$
  for going in direction dir after arriving from direction arr:\\
  \begin{minipage}{0.5\linewidth}
    \begin{align*}
      \pdu = \psu &= \frac{\bup}{b}\\
      \pus = \pds &= \frac{\bside}{b}\\
      \pud = \psd &= \frac{\bdown}{b}\\
    \end{align*}
  \end{minipage}
  \begin{minipage}{0.49\linewidth}
    \begin{align}
      \puu &= \frac{\bup'}{b}\nonumber\\
      \pss &= \frac{\bside'}{b}\label{eq:trans-prob}\\
      \pdd &= \frac{\bdown'}{b}\nonumber\\\nonumber
    \end{align}
  \end{minipage}
  For example, $\pdu$ is the probability for coming from a node below
  and going to one level above.
\end{definition}

The average branching factors are a good basis for the transition probabilities.

\paragraph{Depth distribution}

We are interested in finding a distribution $P_n(l)$ for the probability
of the search being at depth $d$ after having travelled $n$ steps
from the start node.

\begin{definition}[Length-to-depth distribution]
  Let $\pi = v_0, v_1,\dots,v_n$ be a random path starting from the root $v_0$
  and not visiting any node twice.
  (To be precise, the $i+1$st step of the path is made uniformly randomly
  among the neighbours of $v_i$ that are not already in the path.
  If no such neighbour exist, backtrack to the first node where a
  different choice was possible.)

  The \emph{length-to-depth distribution $P_n(l)$} is the
  probability that $\lvl(v_n)=l$.
\end{definition}

The transition probabilities \eqref{eq:trans-prob} define a Markov
chain with transition probabilities:
\[
P =
\begin{pmatrix}
  \pdu & \pds & \pdd \\
  \psu & \pss & \psd \\
  \puu & \pus & \pud
\end{pmatrix}
\]
Integrating over all possible $n$-step transition sequences
of this Markov chain gives the distribution $P_n(l)$.
An approximation of $P_n(l)$ may be obtained by finding the stationary probability
distribution $\pi = (\pup, \pside, \pdown)$ of $P$.
Roughly, $\pup$, $\pside$, and $\pdown$ are the
unconditional probabilities of the search moving upward,
sideward, and downward.
To approximate $P_n(l)$, we consider all combinations of
$n$ step paths so that the final result is $\lvl(v_n)=l$.
This gives for $l\leq n$,
\begin{equation}\label{eq:pnd}
  P_n(l) \approx 
  \sum_{\substack{u+s+d=n\\d-u=l}} {n\choose u,s,d}
  \pup^u\cdot\pside^s\cdot \pdown^{d}
\end{equation}
where $u$, $s$, and $d$ are integers representing the number of upwards,
sidewards, and downwards  number of steps the search takes.
For $l>n$, $P_n(l) = 0$.

\subsection{Depth-to-Depth}
\label{sec:depth-to-depth}

The branching factors also determine how many nodes at depth $d$
are reachable from an average node on level $l$.

\begin{definition}[Depth-to-depth counter]
  Let the \emph{depth-to-depth counter}
  \[
    K(l, d, r) = \E\left[ |\{v : \lvl(v)
      = d, \dist(v, u)\leq r\}| \;\Big|\;  \lvl(u) = l \right]
  \]
  be the average number of nodes on level $d$ reachable in at most $r$ steps
  from a node on level $l$.
  Let the \emph{non-unique depth-to-depth counter} $\bar K(l, d, r)$
  be the average number of paths of length at most $r$
  starting from a node on level $l$ and ending on level $d$.
  (The average, as usual, taken with respect to the equilibrium distribution.)
\end{definition}

To relate the depth-to-depth counters $K(l, d, r)$ and
$\bar K(l, d, r)$ to the branching factors, we introduce some extra notation:
Let $\Seq^{\leq r}(m)$ be the set of sequences $\seq=\{\dir_1,\dots,\dir_k\}$
of length $k\leq r$,
where $\dir_i\in\{\up, \side, \down\}$, $1\leq i\leq k$,
and whose number of down moves are $m$ more than
their number of up moves
$
  |\{i: \dir_i = \down\}| - |\{j: \dir_j = \up\}| = m
$
for $-D\leq m\leq D$.
If $\dir_1\not=\dir_2$, let
$\bl_{\dir_1,\dir_2} = \bl_{\dir_2}$ and $\bg_{dir_1,\dir_2} = \bg_{\dir_2}$.
Finally, we let
$\bg_{\dir,\dir} := \bl_{\dir,\dir} := \bl_{\dir}'$.
This assignment of $\bg_{\dir,\dir}$
may be justified on the grounds that $\bg_{\dir,\dir}$ is
effectively a $\bgdir[,1]$ parameter, and should therefore be equal to
its local counterpart (see discussion following \cref{def:glo-branch}).

\begin{theorem}[Depth-to-depth, general case]
  \label{th:ddc-gen}
  Given that the graph is sufficiently uniform so that the branching
  factors $\bdir$ and $\bgdir$ give a good approximation to the number
  of nodes and number of unique nodes are discovered per level,
  the depth-to-depth counters relates to the branching factors as
  \begin{equation}\label{eq:dtd}
    K(l, d, r) \approx \min\left\{\bgdown^d,
      \sum_{\seq\in\Seq^{\leq r}(d-l)}\prod_{i=0}^{|\seq|-1} \bg_{\dir_i,\dir_{i+1}}\right\}
  \end{equation}
  and
  \begin{equation}\label{eq:dtdb}
    \bar K(l, d, r)
    \approx \sum_{\seq\in\Seq^{\leq r}(d-l)}\prod_{i=0}^{|\seq|-1} \bl_{\dir_i,\dir_{i+1}}.
  \end{equation}
  Here, $\dir_0$ is the direction from which the starting node on
  level $l$ was reached (and empty sums are 0).
\end{theorem}
\todo[inline]{How make clear approximation?}
\begin{proof}
  By definition, the set $\Seq^{\leq r}(d-l)$ includes  the
  different variations of going upwards, sidewards, and
  downwards for at most $r$ steps and ending up $d-l$ steps further down.
  The average branching factors give how many options, on average,
  such a path will have.

  The unique nodes on any level $d$ cannot exceed the number
  of nodes $\bgdown^d$ on this level, which justifies the minimisation in \eqref{eq:dtd}.
  No similar restriction applies to the non-unique count in \eqref{eq:dtdb}.
\end{proof}

Note that the result is only approximate.
For example, the approximation is not perfect when $l$ and $d$
are much smaller than $r$.
In such cases, paths that initially head upward for more than
$l$ steps are not possible.
Although these paths could in principle be excluded from $\Seq^{\leq r}$,
we do not expect this to significantly change the estimate in most cases.

A more efficient approximation is possible when $\bside=0$.

\begin{corollary}[Depth-to-depth, $\bside=0$]\label{th:ddc-spec}
  In addition to the assumptions of \cref{th:ddc-gen},
  assume $\bside=0$. 
  Let $\rextra = r - |n-d|$.
  If $d-n>0$, let $\bg = \bgdown$ and $\bl=\bdown$;
  otherwise let $\bg = \bgup$ and $\bl=\bup$.
  The depth-to-depth counters relate to the branching factors as
  \begin{multline}\label{eq:K-spec2}
    K(l, d, r) \approx \min\bigg\{\bgdown^d,\;\; \bg^{|l-d|} +\\
    +\sum_{m=1}^{\floor{\rextra/2}}
    \sum_{t=1}^{m}
    {m-1 \choose t-1}
    {|l-d|+m \choose t}
    \bg^{|l-d|}(\bgdown\bgup)^{m-t}(\bdown'\bup')^t\bigg\}
  \end{multline}
  and
  \begin{multline}\label{eq:K-spec}
    \bar K(l, d, r) \approx \bl^{|l-d|} +\\
    +\sum_{m=1}^{\floor{\rextra/2}}
    \sum_{t=1}^{m}
    {m-1 \choose t-1}
    {|l-d|+m \choose t}
    \bl^{|l-d|}(\bdown\bup)^{m-t}(\bdown'\bup')^t
  \end{multline}
  when $r\geq |d-n|$
  If $r< |d-n|$, then $K(l, d, r) =\bar K(l, d, r)=0$.
\end{corollary}
The interpretation of $m$ is the number of time steps the search goes in
the ``wrong'' direction, for example heading upwards when the desired
level $d$ is below the starting level $l$.
The interpretation of $t$ is the number of times the direction switches
from upwards-to-downwards-to-upwards or vice versa.

\todo[inline]{add constant for direction}

\begin{proof}[Proof of \cref{th:ddc-spec}]
  The result follows from the more general \cref{th:ddc-gen}.
  Fixating the number of steps $m$ that the search goes in the
  ``wrong'' direction, and the number of switches $t$ between heading
  upwards and downwards, the product simplifies as
  \[
    \prod_{i=0}^{|\seq|-1} \bl_{\dir_i,\dir_{i+1}}
    = \bl^{|l-d|}(\bdown\bup)^{m-t}(\bdown'\bup')^t
  \]
  and similarly for \cref{eq:K-spec2}.
  Note that in \cref{eq:K-spec2}, the local discounted branching factors
  are used in the last factor.

  The first term in \eqref{eq:K-spec} and \eqref{eq:K-spec2}
  accounts for the special case where no direction switches are made,
  i.e.\ $t=0$.
  Then no steps can be taken in the wrong direction, so $m=0$ as well.
\end{proof}

When $\bgup\approx 1$ and $\bgdown\gg 1$, the upper bound will be dominated
by the first term of the sum\footnote{The binomial coefficients grow
subexponentially in the lower argument, ${n \choose k}\leq n^k/k!$.},
yielding the even more easily computed approximation
\begin{equation*}
  K(l, d, r) \approx  \bg^{|l-d|} +
  (|l-d|+\floor{\rextra/2})
  \bg^{|l-d|}(\bgdown\bgup)^{\floor{\rextra/2}}(\bgdown'\bgup')
\end{equation*}
and similarly for $\bar K$ and $\bup$ and $\bdown$.
\todo{maybe not keep this?}


\paragraph{Length-to-depth counter}

Combining the depth-to-depth counters $K$ and $\bar K$
with the length-to-depth distributions $P_n$ gives us the expected number
of nodes reachable on level $d$ when the DFS path length is $n$.

\begin{definition}[Length-to-depth counters]
  For a given radius of search $D'$,
  and depth-do-depth counters $\bar K$ and $K$, let the
  \emph{level-to-depth unique counter} be
  \[
    L(n, d) = \sum_{l=0}^DP_n(l)K(l, d, {D'}-n)
  \]
  and the \emph{level-to-depth non-unique counter} be
  \[
    \bar L(n, d) = \sum_{l=0}^DP_n(l)\bar K(l, d, {D'}-n)
  \]
  for a given path length $n$ and depth $d$, $0\leq d\leq D$
  .
\end{definition}

Assuming accurate depth-to-depth counters and depth distribution,
the level-to-depth counter $L(n, d)$ is the expected number of nodes
reachable on level $d$ after search length $n$, and
$\bar L(n, d)$ counters nodes with repetition
when several paths lead to the same node.

\subsection{Estimating Goal Probabilities}

By solving various instances of a search problem $G$,
we may gather data of the type
\[
  \hat p_l = \frac{\text{number of goals found on level }l}{\text{number of nodes searched on level }l}.
\]
If the level is unknown
(as it usually is when the problem is a graph and not solved completely)
the length-to-depth distribution $P_n(l)$ (\cref{sec:length-to-depth})
can be used to make an estimate of the level $l$.

In this manner, data of type $G\mapsto p_l$ may be gathered
for $0\leq l\leq D$.
Let $\phi_G$ be some \emph{features} of $G$.
The inference problem $\phi_G\mapsto \p$ may be solved with
suitable statistical or machine learning method.
In scenarios where different type of data is available,
different or more advanced estimation techniques may work better.



\section{Grammar Problems}\label{sec:grammar}

We now show how to apply the general theory of \cref{sec:collide} to
two concrete grammar problems.
In these grammar problems, the length-to-depth counters can
be derived analytically, without relying on estimated branching factors
(indeed, the branching factors are not stable in these problems).
As usual, we assume that the goal probability vector $\p$ is given.
This means that \cref{pr:dfs-cb,pr:bfs-cb} can directly be applied,
and their predictions tested (\cref{sec:experimental-verification}).
We only focus on graph search in this section.

A \emph{grammar problem} is a constructive
search problem where nodes
are strings over some finite alphabet $B$,
and the neighbourhood relation is given by a set of production rules.
\emph{Production rules} are mappings $x\to y$, $x,y\in B^*$,
defining how strings may be transformed (for details, see \citet{Hopcroft1979}).
For example, the production rule $S\to Sa$ permits the string $aSa$
to be transformed into $aSaa$.
A grammar problem is defined by a set of production rules, together with
a \emph{starting string} and a \emph{set of goal strings}.
A \emph{solution} is a sequence of production rule applications that transforms
the starting string into a goal string.
Many search problems can be formulated as grammar problems,
with string representations of states modified by production rules.
Their generality makes it \emph{computably undecidable}
whether a given grammar problem has a solution or not.
We here consider a simplified version where the search depth is artificially
limited, and goals are distributed according to a goal probability
vector $\p$.

Grammar problems exhibit two features not present in the complete tree
model.
First, it is possible for branches of the grammar tree to `die'. This
happens if no production rule is applicable to the string of the state.
Second, often the same string can be produced by different sequences
of production rules,
which means that grammar search graphs generally are not trees.

\subsection{Binary Grammar}\label{sec:bg}

The first grammar we consider has only two production rules,
both of which can be applied to any string.

\paragraph{Definition}
Let $\epsilon$ be the empty string.
The \emph{binary grammar} consists of two production rules,
$\epsilon\to a$ and $\epsilon\to b$ over the alphabet $B=\{a,b\}$.
The starting string is the empty string $\epsilon$.
A maximum depth $D$ of the search graph is imposed,
and strings on level $k$ are goals with iid probability $p_k$, $0\leq k\leq D$.
Since the left hand substring of both production rules is the empty string,
both can always be applied at any place to a given string.
The resulting graph is shown in \cref{fig:clustered}.
\begin{figure}
\centering
\includegraphics{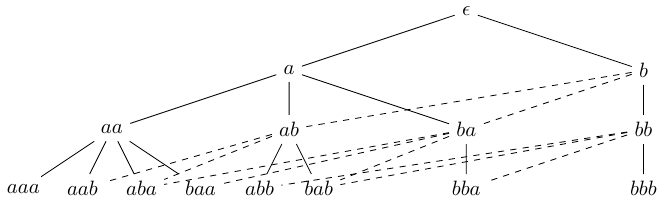}
\caption{Graph of binary grammar problem with max depth $D=3$.
Contiguous lines indicate first discovery by DFS, and dashed
lines indicate rediscoveries.
Nodes further to the right will
have more of their children previously discovered.
}
\label{fig:clustered}
\end{figure}

\paragraph{Analysis}
To get a sense of the induced search graph, the number of children
and parents of a node can be calculated by simple combinatorics.
Consider a node $v$ at level $d$. Its children are reached by either adding
an $a$ or by adding one $b$. Let $\# a$ denote the number of $a$'s in $v$,
and let $\# b$ denote the number of $b$'s in $v$.
Then $\# a+1$ distinct strings can be created by adding a $b$,
and  $\# b+1$ distinct strings can be created by adding an $a$.
In total then, $v$ will have $(\# a+1)+(\# b+1)=d+2$ children,
i.e.\ $\bdown(v)=d+2$ for any node on level $d$.
The number of parents of a node is the number of contiguous
$a^i$ and $b^j$ segments.
For example, $bbaaab$ have three segments $bb$-$aaa$-$b$ and
three parents $b\, aaa\, b$, $bb\, aa\, b$ and $bb\, aaa$.
A parent always differs from a child by the removal of one letter
from one segment,
and within a segment it is irrelevant which letter is removed.

Assuming that the production
rule $\epsilon\to a$ is always used first by DFS,
the first node on level $n$ that DFS reaches in the
binary grammar problem
is $\delta_n=a^n$ for $0\leq n\leq D$.
The following two lemmas derive expressions for \
the length-to-depth counter $\lbg$ and $\bar\lbg$
required by \cref{pr:dfs-cb}. Incidentally, the
number $A_{n,d}$ of level-$d$ $\delta_n$ explorables
(defined in \cref{sec:dfs-cb})
gets an elegant form in the binary grammar problem.

\begin{lemma}[Length-to-depth counter Binary Grammar]
\label{le:lbg}
For $n<d$,
let $\lbg(n,d) = |\{ v : \lvl(v)=d, v\in\descendants(a^n) \}|$  be the number of
nodes reachable from $a^n$, and
let $A_{n,d}=\lbg(n,d)-\lbg(n+1,d)$ be the number of descendants of $a^n$ that
are not descendants of $a^{n+1}$.
Then $\lbg(n,d)=\sum_{i=0}^{d-n}{d \choose i}$,
and $A_{n,d} = {d\choose d-n}$.
\end{lemma}

\begin{proof}
The reachable nodes on level $d$ that we wish to count are $d-n$ levels
below $a^n$.
To reach this level we must add $i\leq d-n$ number of $b$'s and $d-n-i$ number of
$a$'s to $a^n$.
The number of length $d$ strings containing exactly $i$ number of $b$'s
is ${d\choose i}$
(we are choosing positions for the $b$'s non-uniquely with repetition
among $d-i+1$ possible positions).
Summing over $i$, we obtain $\lbg(n,d)=\sum_{i=0}^{d-n}{d \choose i}$,
and $A_{n,d}=\lbg(n,d)-\lbg(n+1,d) = {d\choose d-n}$.
\end{proof}

\begin{lemma}[Non-unique length-to-depth counter Binary Grammar]
  \label{le:lbgt}
  For $n<d$,
  let $\bar\lbg(n,d)$ be the non-unique length-to-depth
  counter for the Binary Grammar,
  i.e.\ the number of paths from $a^n$ to level $d$.
  Then $\bar\lbg(n,d)=\prod_{l=n}^{d-1}(l+2)$.
\end{lemma}

\begin{proof}
  As observed above, nodes on level $l$ have $l+2$ children.
  The number of paths from level $n$ to level $d$ is obtained
  by multiplying the number of options at each step.
\end{proof}

Based on these lemmas,
the expected runtimes of BFS, DFS tree search, and DFS graph search
can be calculated:

\begin{corollary}[BFS runtime on Binary Grammar problem]\label{co:bg-bfs}
The expected BFS search time $\tdfsb(\p)$ in a Binary Grammar Problem of depth
$D$ with goal probabilities $\p=[p_0,\dots,p_D]$ is
\[\tbfsb(\p) = \tbfsc(\p, \lbg).\]
\end{corollary}

\begin{corollary}[DFS graph search runtime on Binary Grammar problem]\label{co:bg-dfs}
The expected DFS search time $\tdfsb(D, \p)$ in a binary grammar problem of depth
$D$ with goal probabilities $\p=[p_0,\dots,p_D]$ is
bounded between $\tdfsbl(D, \p) := \tdfscl(D, \p, \lbg)$ and
$\tdfsbu(D, \p) := \tdfscu(D, \p, \lbg)$, and is approximately
\[\tdfsb(D, \p) := \tdfsc(D, \p, \lbg).\]
\end{corollary}

\begin{corollary}[DFS tree search runtime on Binary Grammar problem]\label{co:bg-dfst}
The expected DFS search time $\tdfsb(D, \p)$ in a binary grammar problem of depth
$D$ with goal probabilities $\p=[p_0,\dots,p_D]$ is
bounded between $\tdfsbl(D, \p) := \tdfscl(D, \p, \lbg, \bar\lbg)$ and
$\tdfsbu(D, \p) := \tdfscu(D, \p, \lbg, \bar\lbg)$, and is approximately
\[\tdfsb(D, \p) := \tdfsc(D, \p, \lbg, \bar\lbg).\]
\end{corollary}

\begin{proof}[Proof of \cref{co:bg-bfs,co:bg-dfs,co:bg-dfst}]
Direct application of
\cref{le:lbg,le:lbgt}, and \cref{pr:bfs-cb,pr:dfs-ts,pr:dfs-cb}
respectively.
\end{proof}
The estimates are plotted for a single goal level in
\cref{fig:decision-boundaries,fig:search-time}.

\subsection{Full Grammar}\label{sec:fg}

Our second grammar builds on a larger set of production
rules that can move a start symbol $S$ around, and elicit
the letters $a$ and $b$ from $S$.

\paragraph{Definition}
The \emph{full grammar problem} has alphabet $B=\{S, a, b\}$ and start string $S$.
The \emph{production rules} are $S\to\epsilon$ (with $\epsilon$
the empty string) plus the \emph{adding rules}
$S\to Sa$,
$S\to aS$,
$S\to Sb$,
$S\to bS$,
and the \emph{moving rules}
$Sa\to aS$,
$aS\to Sa$,
$Sb\to bS$, and
$bS\to Sb$.
Only $S$-less strings can be goal nodes. As usual, a maximum depth $D$
and a goal probability vector $\p=[p_0, \dots, p_D]$ are given.

\paragraph{Analysis}
For simplified analysis,
we will abuse notation the following way.
We will consider $S$-less nodes to be one level higher than they
actually are. For example, $a$ would normally be on level 2
(e.g.\ reached by the path $S\to Sa$, $S\to\epsilon$),
but we will consider it to be on level 1.
A slight modification of BFS and DFS makes them always check the $S$-less child
first (which is always child-less in turn), which means the change will only slightly
affect search time.
We will still consider $\delta_n=Sa^n$ whenever $S\to Sa$ is among the production
rules, however.

The search graph of the full grammar problem is shown in
\cref{fig:unclustered} (edges induced by moving rules are not shown).
Since there are four adding rules that can be applied to each node,
each node will have four children.
Typically, when we move further to the right in the tree, more children
will already have been discovered.

\begin{figure}
  \centering
  \includegraphics{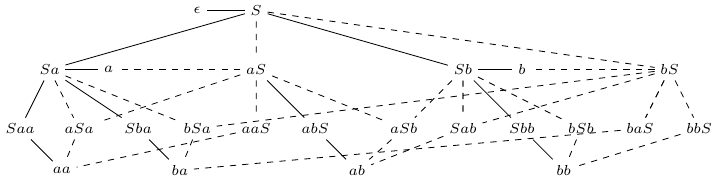}
\caption{Search graph for the Full Grammar problem until level 2.
Connections induced by moving rules are not displayed.
Contiguous lines indicate the first discovery of a child by DFS and
dashed lines indicate rediscoveries.
}
\label{fig:unclustered}
\end{figure}

The full grammar problem can be analysed  by a reduction to a
binary grammar problem with the same parameters $D$ and $\p$.
Assign to each string $v$ of the binary grammar problem the set of strings
that only differ from $v$ by (at most) an extra $S$.
We call such sets \emph{node clusters}. 
For example, $\{a, Sa, aS\}$ constitutes the node cluster corresponding
to $a$.
Due to the abusing of levels for the $S$-less strings, all members
of a cluster appear on the same level in the full grammar problem (the
level is equal to the number of $a$'s and $b$'s).
The level is also the same as the corresponding string in the
binary grammar problem.

\begin{lemma}[Length-to-depth counter Full Grammar]\label{le:lbf}
For every $n$, $d$, $n\leq d$, the length-to-depth counter $\lfg$ of the
full grammar problem is $\lfg(n, d) = (d+2)\lbg(n,d)$.
\end{lemma}
\begin{proof}
$\lbg(n, d)$ counts the level $d$ descendants of $a^n$ in the binary
grammar problem (BGP), and
$\lfg(n, d)$ counts the level $d$ descendants of $Sa^n$ in the full
grammar problem (FGP). The node
$u$ is a child of $v$ in BGP iff the members of the $u$ node cluster
are descendants of $Su$.
Therefore the node clusters on level $d$ descending from $Sa^n$ in
FGP correspond to the BGP nodes descending from $a^n$.
At level $d$, each node cluster contains $d+2$ nodes.
\end{proof}

The non-unique length-to-depth counter $\bar\lfg$ can be approximated
from the local branching factors $\bup=1$, $\bside\approx 2$, $\bdown\approx 4$
as described in \cref{sec:graph-param}.
Analogously to the Binary Grammar case, the length-to-depth counters
give us the expected runtime of BFS and DFS:

\begin{corollary}[Expected BFS runtime on Full Grammar]\label{co:fg-bfs}
The expected BFS search time $\tdfsf(\p)$ in a full grammar problem of depth
$D$ with goal probabilities $\p=[p_0,\dots,p_D]$ is
\[\tbfsf(\p) := \tbfsc(\p, \lfg).\]
\end{corollary}

\begin{corollary}[Expected DFS graph search runtime on Full Grammar]\label{co:fg-dfs}
The expected DFS search time $\tdfsf(D, p)$ in a full grammar problem of depth
$D$ with goal probabilities $\p=[p_0,\dots,p_D]$ is
bounded between $\tdfsfl(D,\p):=\tdfscl(D,\p,\lfg)$ and
$\tdfsfu(D,\p):=\tdfscu(D,\p,\lfg)$, and is approximately
\[\tdfsf(D,\p) := \tdfsc(D, \p, \lfg).\]
\end{corollary}
\begin{proof}[Proof of \cref{co:fg-bfs,co:fg-dfs}]
Direct application of
\cref{le:lbf}, and \cref{pr:bfs-cb,pr:dfs-cb}
respectively.
\end{proof}

\cref{co:fg-bfs,co:fg-dfs,co:bg-bfs,co:bg-dfs,co:bg-dfst} show that it is
possible to estimate BFS and DFS expected runtime by
analytically deriving the length-to-depth counter.
The next section verify the predictions empirically.
Among other things, it shows that the DFS bounds can
be used to predict expected runtime reasonably well.

\section{Experimental Results}\label{sec:experimental-verification}

To verify the analytical results, we have implemented the models
of \cref{sec:sgl,sec:mgl,sec:cb,sec:grammar,sec:graph-param}
in Python 3 using the \texttt{graph-tool} package \citep{Peixoto2015}.%
\footnote{Source code for the experiments is available at \url{http://tomeveritt.se}.}

\paragraph{Gaussian Binary Tree}

To develop a concrete instance of the multiple goal level model
we consider the special case of \emph{Gaussian goal probability vectors},
with two parameters $\mu$ and $\sigma^2$.
For a given depth $D$, the goal probabilities are given by
\begin{equation*}
  p_i = \min\left\{ \frac{1}{20\sqrt{\sigma^2}}e^{(i-\mu)^2/\sigma^2},\; \frac{1}{2} \right\}.
\end{equation*}
The parameter $\mu\in [0,D]\cap\SetN$ is the \emph{goal peak}, and
the parameter $\sigma^2\in\SetR^+$ is the \emph{goal spread}.
The factor $1/20$ is arbitrary, and chosen to give an interesting dynamics
between searching depth-first and breadth-first.
No $p_i$ should be greater than $1/2$, in order to (roughly)
satisfy the assumption of \cref{pr:bfs-mgl}.
We call this model the \emph{Gaussian binary tree}.

An important feature of the Gaussian goal probabilities are that
they decay equally fast both upward and downward from the goal peak
level $\mu$. An arbitrary node situated $k$ levels above the goal peak
has the same probability of being a goal as an arbitrary
node situated $k$ levels below the peak, for any
$k\in\{0,\dots,\min(\mu, D-\mu)\}$.

\subsection{Runtimes and Decision Boundaries}

\paragraph{Expected Runtime Plots}

\begin{figure}
  \rotatebox[origin=c]{90}{expected search time}
  \begin{subfigure}{.5\textwidth}
    \centering
    \includegraphics{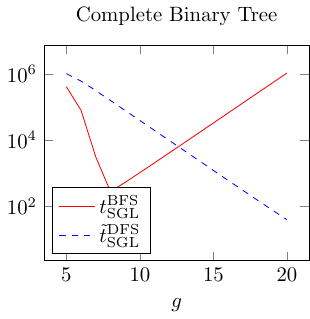}
  \end{subfigure}
  \begin{subfigure}{.5\textwidth}
    \centering
    \includegraphics{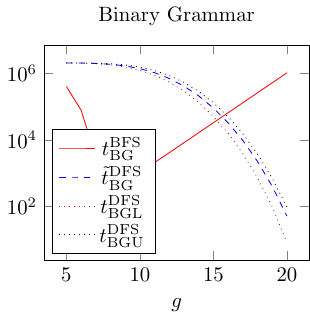}
  \end{subfigure}
  \caption{The expected search time of BFS and DFS graph search as a function of
    a single goal level $g$ with goal probability $p_g=0.05$ in a tree of depth $D=20$.
    (The part hidden by the legend is identical for both plots.)
    BFS has the advantage when the goal is in the higher regions of the graph,
    although at first the probability that no goal exists heavily influences
    both BFS and DFS search time.
    The greater connectivity of the graph in the binary grammar problem
    permits DFS to spend more time in the lower regions before backtracking,
    compared to the complete binary tree analysed in
    the previous section.
    This penalises DFS runtime when the goal is not in the very lowest regions of
    the tree. 
  }
  \label{fig:search-time}
\end{figure}

It is a natural exercise to plot the expected runtime as a
function of the involved parameters.
\cref{fig:search-time} plots the expected runtimes for a single
goal level in both a binary tree and a binary grammar.
BFS is better for goals close to the root and DFS graph search better
when the goals are farther from the root in both models, as expected.
The initially high value of BFS depends on the high likelihood
of there being no goal at all when the goal level is close to the
root and only contain a few nodes.
When there are no goals, both BFS and DFS will search the entire space.

More surprising is the fact that the crossover occurs later
in the more connected graph of the Binary Grammar.
The reason is that DFS can spend longer time in the very lowest
regions of the graph before backtracking due to the higher
connectivity (compare \cref{fig:bfs-vs-dfs,fig:bfs-vs-dfs-graph}).

\paragraph{Decision Boundaries}

\begin{figure}
  \centering
  \includegraphics{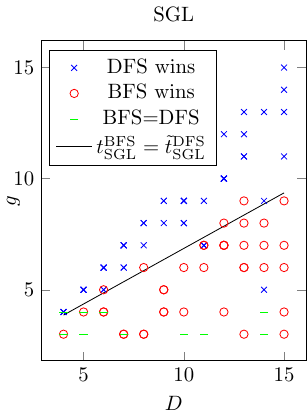}
  \includegraphics{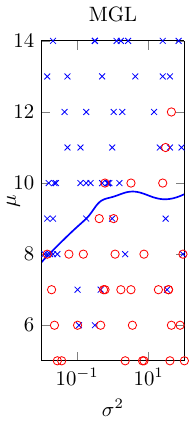}
  \includegraphics{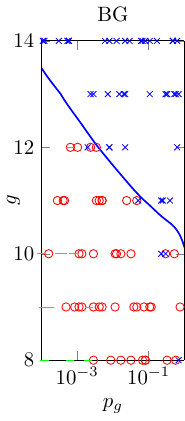}
  \caption{
    The left graph shows the decision boundary of \cref{pr:bfs-vs-dfs-sgl}
    for the single goal level tree,
    The scattered points come from 100 empirical outcomes
    of BFS and DFS graph search times according
    to the varied parameters $g\in \{3,\dots, D\}$ and $D \in \{4,\dots, 15\}$.
    The decision boundary gets $79\%$ of the winners correct.
    The middle graph shows the decision boundary for the Gaussian tree
    given by \cref{pr:dfs-mgl,pr:bfs-mgl}.
    The scattered points are based on 100 independently generated trees
    with depth $D=14$ and uniformly sampled parameters $\mu\in\{5,\dots,14\}$
    and $\log(\sigma^2)\in [-2,2]$.
    The most deciding feature is the goal peak $\mu$, but DFS also benefits
    from a smaller $\sigma^2$.
    The decision boundary gets $74\%$ of the winners correct.
    %
    The right graph shows the decision boundary predicted by
    \cref{co:bg-bfs,co:bg-dfs} for the binary grammar.
    The scattered points are based on 100 independently
    generated binary grammar problems of depth $D=14$
    with uniformly sampled (single) goal level $g\in \{8,\dots,14\}$
    and $\log(p_g)\in[-4,0]$.
    DFS benefits from a deeper goal level and higher goal probability
    compared to BFS.
    The decision boundary gets $87\%$ of the instances correct.
    %
    Most ties (green dashes) are due to no goal being present.
  }
  \label{fig:decision-boundaries}
\end{figure}

By comparing the expected runtimes of
\cref{pr:tdfs-sgl,pr:tbfs-sgl};
of \cref{pr:dfs-mgl,pr:bfs-mgl}; and
of \cref{co:bg-bfs,co:bg-dfs}, decision boundaries of which
algorithm is the better can be obtained.
\cref{fig:decision-boundaries} shows these boundaries together
with actual outcomes of which algorithm was faster on
randomly generated instances with the given parameters.

The single goal level plot shows that BFS likes goals
closer to the root, and that decision boundary given by
\cref{pr:bfs-vs-dfs-sgl} predicts the winner almost perfectly.
It only fails in instances very close to the boundary.

In the decision boundary for the tree with multiple goal levels,
we plot the decision boundary as a function of goal peak
and goal spread in the Gaussian binary tree model.
In addition to finding that BFS prefers a higher goal peak (lower $\mu$),
we find that BFS also benefits relative to DFS from a greater
spread $\sigma$.
We can explain this result in light of \cref{pr:bfs-vs-dfs-sgl}.
Roughly, a level is \emph{relevant} only if it has high enough
goal probability that there is a substantial chance the level has a goal.
For the relevant levels, a high goal probability $(\geq 0.2)$, will
make the level give the same expected search time to both BFS and DFS
if it is located midway between 0 and $D$.
For smaller goal probabilities, e.g.\ $p_i \approx 2^{-2k}$,
level $i$ will benefit BFS more than DFS if $i< D/2-k$.
Now, when the spread is low, only a single level is relevant (the mean level $\mu$)
and it has high goal probability (as much as $p_\mu=1/2$).
When the spread increases, BFS is benefited in two ways:
First, the probability $p_\mu$ decreases, which benefits BFS
according to \cref{pr:bfs-vs-dfs-sgl}.
Second, additional levels $\mu-1$ and $\mu+1$ become relevant.
As their goal probabilities are small, BFS will benefit from both of
those levels unless $\mu$ is significantly closer to $D$ than to 0.
The prediction accuracy is slightly lower than in the single
goal level case, plausibly due to
the increased random component of the goal model.

Finally, with the binary grammar, we experiment with adjusting
the goal probability and the goal level.
It can be seen that DFS clearly benefits from a higher goal
probability to a much larger extent than BFS.
Increasing the goal probability by a factor 10 shifts the
advantage about as much as shifting the goal level by 1.
It is unsurprising that DFS benefits from a high goal probability,
since when the goal probability is high, a random trajectory down
through the graph is likely to hit a goal fast.

Overall, the decision boundaries largely match empirical outcomes.

\subsection{Empirical Averages}

The data reported in \cref{tab:sgl,tab:mgl,tab:bg}
is based on an average over 1000 independently generated search
problems with depth $D=14$.
\begin{itemize}
\item The first number in each box is the empirical average,
\item the second number is the analytical estimate from previous sections, and
\item the third number is the percentage error of the analytical estimate.
\end{itemize}

For certain parameter settings, there is only a small chance
($<10^{-3}$) that there are no goals. In such circumstances, all
1000 generated search graphs typically inhabit a goal,
and so the empirical search times will be comparatively small.
However, since a tree of depth 14 has about $2^{15}\approx 3\cdot 10^{5}$ nodes
(and a search algorithm must search through all of them in case there is no goal),
the rarely occurring event of no goal can still influence the
\emph{expected} search time substantially.
To avoid this sampling problem, we have ubiquitously discarded all
instances where no goal is present,
and compared the resulting averages to the
analytical expectations \emph{conditioned on at least one goal being present}.
These modified analytical expectations are obtained by removing
the term corresponding to `no goal' and renormalising the probabilities.
Details are discussed in connections to the results above.
Since the calculation of the probability that no goal exists
and the search time when no goal exists are both uncontroversial,
there is limited reason to verify these parts experimentally.

\todo{one reviewer complains about this data missing}

\paragraph{Complete Tree}

\begin{table}
\begin{subtable}{0.5\textwidth}
\centering
\begin{tabular}{|l|r|r|r|}
\hline
$g \backslash p_g$ & 0.001 & 0.01 & 0.1 \\ \hline
\num{5} & & \num{46} & \num{40} \\
& & \num[math-rm=\mathit]{47} & \num[math-rm=\mathit]{40} \\
& & \SI[round-precision=1]{0.7}{\percent} & \SI[round-precision=1]{0.4}{\percent} \\
\hline
\num{8} & \num{369} & \num{332} & \num{264} \\
& \num[math-rm=\mathit]{378} & \num[math-rm=\mathit]{333} & \num[math-rm=\mathit]{265} \\
& \SI[round-precision=2]{2.3}{\percent} & \SI[round-precision=1]{0.3}{\percent} & \SI[round-precision=1]{0.2}{\percent} \\
\hline
\num{11} & \num{2747} & \num{2143} & \num{2056} \\
& \num[math-rm=\mathit]{2744} & \num[math-rm=\mathit]{2147} & \num[math-rm=\mathit]{2057}  \\
& \SI[round-precision=1]{0.1}{\percent} & \SI[round-precision=1]{0.2}{\percent} & \SI[round-precision=1]{0.0}{\percent} \\
\hline
\num{14} & \num{17364.820} & \num{16482.680} & \num{16392.680}  \\
& \num[math-rm=\mathit]{17383.000} & \num[math-rm=\mathit]{16483.000} & \num[math-rm=\mathit]{16393.000} \\
& \SI[round-precision=1]{0.1}{\percent} & \SI[round-precision=1]{0.0}{\percent} & \SI[round-precision=1]{0.0}{\percent} \\
\hline
\end{tabular}
\caption{BFS single goal level}
\label{tab:sgl-bfs}
\end{subtable}
\begin{subtable}{0.5\textwidth}
\centering
\begin{tabular}{|l|r|r|r|r|r|r|}
\hline
$g \backslash p_g$ & 0.001 & 0.01 & 0.1 \\ \hline
\num{5} & & \num{14678} & \num{8205} \\
& & \num[math-rm=\mathit]{14998} & \num[math-rm=\mathit]{8052} \\
& & \SI[round-precision=2]{2.1999999999999997}{\percent} & \SI[round-precision=2]{1.9}{\percent} \\
\hline
\num{8} & \num{14533} & \num{9832} & \num{1104} \\
& \num[math-rm=\mathit]{15623.370} & \num[math-rm=\mathit]{9966.700} & \num[math-rm=\mathit]{1154.000} \\
& \SI[round-precision=2]{7.5}{\percent} & \SI[round-precision=2]{1.4}{\percent} & \SI[round-precision=2]{4.5}{\percent} \\
\hline
\num{11} & \num{11200} & \num{1534} & \num{152} \\
& \num[math-rm=\mathit]{11138} & \num[math-rm=\mathit]{1586} & \num[math-rm=\mathit]{146} \\
& \SI[round-precision=1]{0.5}{\percent} & \SI[round-precision=2]{3.4000000000000004}{\percent} & \SI[round-precision=2]{4.1000000000000005}{\percent} \\
\hline
\num{14} & \num{1971} & \num{208} & \num{30} \\
& \num[math-rm=\mathit]{2000} & \num[math-rm=\mathit]{200} & \num[math-rm=\mathit]{20} \\
& \SI[round-precision=2]{1.4000000000000001}{\percent} & \SI[round-precision=2]{4.2}{\percent} & \SI[round-precision=2]{34.599999999999994}{\percent} \\
\hline
\end{tabular}
\caption{DFS single goal level}
\end{subtable}
\caption{BFS and DFS performance in the single goal level model
with depth $D=14$,
where $g$ is the goal level and $p_g$ the goal probability.
Each box contains empirical average/\emph{analytical expectation}/error percentage.
}
\label{tab:sgl}
\end{table}

The accuracy of the predictions of \cref{pr:tbfs-sgl,pr:tdfs-sgl}
are shown in \cref{tab:sgl}, and the accuracy of
\cref{pr:dfs-mgl,pr:bfs-mgl} in \cref{tab:mgl}.
The relative error is always small for BFS ($<10\%$).
For DFS the error is generally within $20\%$, except when the
search time is small ($<35$ nodes are explored),
in which case the absolute error is always small.
These boundary plots show that the analysis generally
predicts the correct BFS vs.\ DFS winner.

As discussed in \cref{sec:sgl}, our BFS results can be compared with
the worst case IDA* result $2^{g+2}$ by \citet{Korf2001}.
Comparing \citeauthor{Korf2001}'s results to a doubling
of the empirical averages in \cref{tab:sgl-bfs} still yields that
\citeauthor{Korf2001}'s predictions are 33-50\% overestimates
compared to empirical outcomes.
This is unsurprising given that \citeauthor{Korf2001}'s estimates
are intended for the worst case.
We did not find
a natural way of adapting \citeauthor{Korf2001}'s results to the
multiple goal level scenarios.

\begin{table}
\begin{subtable}{0.5\textwidth}
\centering
\begin{tabular}{|l|r|r|r|r|}
\hline
\!\!$\mu\backslash \sigma$\!\!\!\!\!\! & 0.1 & 1 & 10 & 100\\ \hline
5 & \num{37} & \num{43} & \num{90} & \num{225}\\
 & \num[math-rm=\mathit]{37} & \num[math-rm=\mathit]{41} & \num[math-rm=\mathit]{83} & \num[math-rm=\mathit]{210}\\
 & \SI[round-precision=1]{0.537056928034}{\percent}  & \SI[round-precision=2]{5.02857142857}{\percent}  & \SI[round-precision=2]{7.86838340486}{\percent}  & \SI[round-precision=2]{6.36577673138}{\percent} \\
\hline
8 & \num{261} & \num{171} & \num{119} & \num{211}\\
 & \num[math-rm=\mathit]{261} & \num[math-rm=\mathit]{173} & \num[math-rm=\mathit]{119} & \num[math-rm=\mathit]{210}\\
 & \SI[round-precision=1]{0.0}{\percent}  & \SI[round-precision=1]{0.895973935304}{\percent}  & \SI[round-precision=1]{0.167238063383}{\percent}  & \SI[round-precision=1]{0.457655107337}{\percent} \\
\hline
11 & \num{2048} & \num{952} & \num{303} & \num{249}\\
 & \num[math-rm=\mathit]{2049} & \num[math-rm=\mathit]{952} & \num[math-rm=\mathit]{304} & \num[math-rm=\mathit]{247}\\
 & \SI[round-precision=1]{0.0}{\percent}  & \SI[round-precision=1]{0.0}{\percent}  & \SI[round-precision=1]{0.345497022145}{\percent}  & \SI[round-precision=1]{0.781688447046}{\percent} \\
\hline
\!14\! & \num{16210} & \num{5159} & \num{968} & \num{332}\\
 & \!\!\num[math-rm=\mathit]{16152} & \num[math-rm=\mathit]{5136} & \num[math-rm=\mathit]{960} & \num[math-rm=\mathit]{329}\\
 & \SI[round-precision=1]{0.355450821695}{\percent}  & \SI[round-precision=1]{0.448882445518}{\percent}  & \SI[round-precision=1]{0.816710031801}{\percent}  & \SI[round-precision=1]{0.940335276092}{\percent} \\
\hline
\end{tabular}
\caption{BFS multi goal level}
\end{subtable}
\begin{subtable}{0.5\textwidth}
\centering
\begin{tabular}{|l|r|r|r|r|}
\hline
\!\!$\mu\backslash \sigma$\!\!\!\! & 0.1 & 1 & 10 & 100\\ \hline
5 & \num{5374} & \num{8572} & \num{3404} & \num{385}\\
 & \num[math-rm=\mathit]{5949} & \num[math-rm=\mathit]{10073} & \num[math-rm=\mathit]{3476} & \num[math-rm=\mathit]{379}\\
 & \SI[round-precision=2]{10.6936120409}{\percent}  & \SI[round-precision=2]{17.5159993561}{\percent}  & \SI[round-precision=2]{2.11729876263}{\percent}  & \SI[round-precision=2]{1.73651608221}{\percent} \\
\hline
8 & \num{677} & \num{1233} & \num{454} & \num{252}\\
 & \num[math-rm=\mathit]{743} & \num[math-rm=\mathit]{1259} & \num[math-rm=\mathit]{473} & \num[math-rm=\mathit]{259}\\
 & \SI[round-precision=2]{9.7884339982}{\percent}  & \SI[round-precision=2]{2.07519333344}{\percent}  & \SI[round-precision=2]{4.18408200977}{\percent}  & \SI[round-precision=2]{2.89740341989}{\percent} \\
\hline
11 & \num{97.38} & \num{168} & \num{117} & \num{210}\\
 & \num[math-rm=\mathit]{92} & \num[math-rm=\mathit]{157} & \num[math-rm=\mathit]{106} & \num[math-rm=\mathit]{211}\\
& \SI[round-precision=2]{4.54918874512}{\percent}  & \SI[round-precision=2]{6.3541170871}{\percent}  & \SI[round-precision=2]{9.10329558035}{\percent}  & \SI[round-precision=1]{0.799961906576}{\percent} \\
\hline
14 & \num{24} & \num{43} & \num{81} & \num{213}\\
 & \num[math-rm=\mathit]{11} & \num[math-rm=\mathit]{33} & \num[math-rm=\mathit]{74} & \num[math-rm=\mathit]{205}\\
 & \SI[round-precision=2]{51.5833333333}{\percent}  & \SI[round-precision=2]{24.1816505302}{\percent}  & \SI[round-precision=2]{8.91743119266}{\percent}  & \SI[round-precision=2]{4.02584027713}{\percent} \\
\hline
\end{tabular}
\caption{DFS multi goal level}
\end{subtable}
\caption{BFS and DFS performance in Gaussian binary trees with depth $D=14$.
Each box contains empirical average/\emph{analytical expectation}/error percentage.
}
\label{tab:mgl}
\end{table}

\paragraph{Grammar}
The binary grammar model of \cref{sec:bg}
serves to verify the general estimates of
\cref{pr:dfs-cb,pr:bfs-cb}.
The results are shown in \cref{tab:bg}.
The estimates for BFS are accurate
as usual
($<3\%$ error).
With few exceptions, the lower and the upper bounds $\tdfsbl$ and $\tdfsbu$
of  \cref{co:bg-dfs} for DFS differ by at most $50\%$ on the
respective sides from the true (empirical) average.
The arithmetic mean $\tdfsb$ often
give surprisingly accurate predictions ($<4\%$) except
when $\tdfsbl$ and $\tdfsbu$ leave wide margins as to the
expected search time (when $g=14$, the margin is
up to $84\%$ downwards and $125\%$ upwards).
Even then, the $\tdfsb$ error remains within $30\%$.

\begin{table}
\begin{subtable}{0.5\textwidth}
\centering
\begin{tabular}{|l|r|r|r|}
\hline
$g\backslash p_g$ & 0.001 & 0.01 & 0.1\\ \hline
\num{5} & & \num{47} & \num{41} \\
& & \num[math-rm=\mathit]{47} & \num[math-rm=\mathit]{40} \\
& & \SI[round-precision=1]{0.2}{\percent} & \SI[round-precision=2]{1.7000000000000002}{\percent} \\
\hline
\num{8} & \num{376} & \num{332} & \num{266} \\
& \num[math-rm=\mathit]{378} & \num[math-rm=\mathit]{333} & \num[math-rm=\mathit]{265} \\
& \SI[round-precision=1]{0.6}{\percent} & \SI[round-precision=1]{0.4}{\percent} & \SI[round-precision=1]{0.3}{\percent} \\
\hline
\num{11} & \num{2751} & \num{2145} & \num{2058} \\
& \num[math-rm=\mathit]{2744} & \num[math-rm=\mathit]{2147} & \num[math-rm=\mathit]{2057} \\
& \SI[round-precision=1]{0.3}{\percent} & \SI[round-precision=1]{0.1}{\percent} & \SI[round-precision=1]{0.0}{\percent} \\
\hline
\num{14} & \num{17372} & \num{16479} & \num{16394} \\
& \num[math-rm=\mathit]{17383} & \num[math-rm=\mathit]{16483} & \num[math-rm=\mathit]{16393} \\
& \SI[round-precision=1]{0.1}{\percent} & \SI[round-precision=1]{0.0}{\percent} & \SI[round-precision=1]{0.0}{\percent} \\
\hline
\end{tabular}
\caption{BFS $\tbfsb$}
\end{subtable}
\begin{subtable}{0.5\textwidth}
\centering
\begin{tabular}{|l|r|r|r|}
\hline
$g\backslash p_g$ & 0.001 & 0.01 & 0.1\\ \hline
\num{5} & & \num{30915} & \num{27837} \\
& & \num[math-rm=\mathit]{31365} & \num[math-rm=\mathit]{30186} \\
& & \SI[round-precision=2]{1.5}{\percent} & \SI[round-precision=2]{8.4}{\percent} \\
\hline
\num{8} & \num{27999} & \num{25157} & \num{15494} \\
& \num[math-rm=\mathit]{27407} & \num[math-rm=\mathit]{24421} & \num[math-rm=\mathit]{15203} \\
& \SI[round-precision=2]{2.1}{\percent} & \SI[round-precision=2]{2.9000000000000004}{\percent} & \SI[round-precision=2]{1.9}{\percent} \\
\hline
\num{11} & \num{17284} & \num{5932} & \num{1815} \\
& \num[math-rm=\mathit]{16787} & \num[math-rm=\mathit]{5806} & \num[math-rm=\mathit]{1788} \\
& \SI[round-precision=2]{2.9000000000000004}{\percent} & \SI[round-precision=2]{2.1}{\percent} & \SI[round-precision=2]{1.5}{\percent} \\
\hline
\num{14} & \num{1304} & \num{122} & \num{26} \\
& \num[math-rm=\mathit]{1522} & \num[math-rm=\mathit]{165} & \num[math-rm=\mathit]{20} \\
& \SI[round-precision=2]{16.7}{\percent} & \SI[round-precision=2]{34.8}{\percent} & \SI[round-precision=2]{21.6}{\percent} \\
\hline
\end{tabular}
\caption{Average DFS $\tdfsb$}
\end{subtable}
\begin{subtable}{0.5\textwidth}
\centering
\begin{tabular}{|l|r|r|r|}
\hline
$g\backslash p_g$ & 0.001 & 0.01 & 0.1\\ \hline
\num{5} & & \num{30915} & \num{27837} \\
& & \num[math-rm=\mathit]{30711} & \num[math-rm=\mathit]{29080} \\
& & \SI[round-precision=1]{0.7000000000000001}{\percent} & \SI[round-precision=2]{4.5}{\percent} \\
\hline
\num{8} & \num{27999} & \num{25157} & \num{15494} \\
& \num[math-rm=\mathit]{25737} & \num[math-rm=\mathit]{22151} & \num[math-rm=\mathit]{12072} \\
& \SI[round-precision=2]{8.1}{\percent} & \SI[round-precision=2]{11.899999999999999}{\percent} & \SI[round-precision=2]{22.1}{\percent} \\
\hline
\num{11} & \num{17284} & \num{5932} & \num{1815} \\
& \num[math-rm=\mathit]{14164} & \num[math-rm=\mathit]{3822} & \num[math-rm=\mathit]{919} \\
& \SI[round-precision=2]{18.099999999999998}{\percent} & \SI[round-precision=2]{35.6}{\percent} & \SI[round-precision=2]{49.4}{\percent} \\
\hline
\num{14} & \num{1304} & \num{122} & \num{26} \\
& \num[math-rm=\mathit]{809} & \num[math-rm=\mathit]{54} & \num[math-rm=\mathit]{4} \\
& \SI[round-precision=2]{38.0}{\percent} & \SI[round-precision=2]{55.7}{\percent} & \SI[round-precision=2]{84.39999999999999}{\percent} \\
\hline
\end{tabular}
\caption{Lower DFS $\tdfsbl$}
\end{subtable}
\begin{subtable}{0.5\textwidth}
\centering
\begin{tabular}{|l|r|r|r|}
\hline
$g\backslash p_g$ & 0.001 & 0.01 & 0.1\\ \hline
\num{5} & & \num{30915} & \num{27837} \\
& & \num[math-rm=\mathit]{32019} & \num[math-rm=\mathit]{31293} \\
& & \SI[round-precision=2]{3.5999999999999996}{\percent} & \SI[round-precision=2]{12.4}{\percent} \\
\hline
\num{8} & \num{27999} & \num{25157} & \num{15494} \\
& \num[math-rm=\mathit]{29075} & \num[math-rm=\mathit]{26690} & \num[math-rm=\mathit]{18335} \\
& \SI[round-precision=2]{3.8}{\percent} & \SI[round-precision=2]{6.1}{\percent} & \SI[round-precision=2]{18.3}{\percent} \\
\hline
\num{11} & \num{17283} & \num{5932} & \num{1815} \\
& \num[math-rm=\mathit]{19411} & \num[math-rm=\mathit]{7789} & \num[math-rm=\mathit]{2657} \\
& \SI[round-precision=2]{12.3}{\percent} & \SI[round-precision=2]{31.3}{\percent} & \SI[round-precision=2]{46.400000000000006}{\percent} \\
\hline
\num{14} & \num{1304} & \num{122} & \num{25} \\
& \num[math-rm=\mathit]{2236} & \num[math-rm=\mathit]{275} & \num[math-rm=\mathit]{36} \\
& \SI[round-precision=2]{71.5}{\percent} & \SI[round-precision=3]{125.29999999999998}{\percent} & \SI[round-precision=2]{41.099999999999994}{\percent} \\
\hline
\end{tabular}
\caption{Upper DFS $\tdfsbu$}
\end{subtable}
\caption{
BFS and DFS performance in binary grammars of depth $D=14$.
Empirical DFS performance is compared to the upper and lower bounds
of \cref{co:bg-dfs}, as well as their arithmetic average.
In these experiments, goals are distributed on a single goal level $g$
with goal probability $p_g$.
The BFS estimates $\tbfsb$ are highly accurate, and the averaged DFS estimates $\tdfsb$ are mostly accurate.
Each box contains empirical average/\emph{analytical expectation}/error percentage.
}
\label{tab:bg}
\end{table}

\subsection{N-Puzzle}
\label{sec:n-puzzle}

\begin{figure}
  \centering
  \vspace{-0.5cm}
  \begin{TAB}(e,2em,2em){|c:c:c|}{|c:c:c|}
    6 &  & 8\\
    3 & 4 & 7\\
    2 & 5 & 1\\
  \end{TAB}
  \caption{8-Puzzle.
    At any stage, the empty tile may be swapped
    with an adjacent tile.
    The goal is to sort the tiles with the empty tile
    at the bottom right.
  }
  \label{fig:n-puzzle}
\end{figure}

In this section, we apply the theory of \cref{sec:cb,sec:graph-param}
to the 8-Puzzle problem (\cref{fig:n-puzzle}),
estimating expected search time from a local sample.
We focus on evaluating BFS (graph search)
and DFS tree search for this problem,
as DFS graph search consistently cut itself off from significant
portions of the 8-puzzle search space
(see discussion in \cref{sec:graph-algorithms},
including \cref{fig:dfs-cut-off} on \cpageref{fig:dfs-cut-off}).

\paragraph{Local branching factors}

The 8-Puzzle appears to approximately satisfy uniformity \cref{as:uniformity},
as can be seen in \cref{fig:branching-factors}.
\begin{figure}
  \centering
  \includegraphics{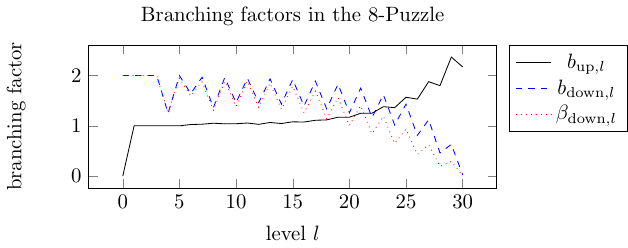}
  \caption{
    Average branching factors as a function of depth.
    Here $\bup[,l]$ is short for average the average value of
    $\bup(v)$ given that $\lvl(v)=l$,
    and similarly for $\bdown[,l]$.
    The branching factors roughly satisfy the uniformity assumption
    up until level 22.
    The majority of the nodes of the 8-puzzle are on level
    22 or above.
    Note also that the global branching factor $\bgdown$ is slightly
    lower than the local branching factor $\bdown$, as expected.
  }
  \label{fig:branching-factors}
\end{figure}
\todo{maybe change starting node?}
Running BFS up until depth 9 and using the average from levels
6 to 9, we find that
\begin{itemize}\label{list:param}
\item $\bup \approx 1.035$
\item $\bside = 0$ (due to invariants in the N-Puzzle,
  different nodes can only be reached in even and odd number of steps)
\item $\bdown \approx 1.80$ 
\item $\bgup \approx \bup$ (the data was insufficient to get a better estimate)
\item $\bgdown \approx 1.66$
\end{itemize}

Despite using levels a few steps away from the start, the
parameters vary somewhat depending on whether the empty tile
started in a corner, in the middle of an edge, or in the middle.
We use a weighted average according to the distribution of a randomly
sampled problem, with the middle edge and corner cases having
relative weight 4 each, and the middle case having relative weight 1.

The branching factors are core to our theory.
They allow us to approximate the length-to-depth distribution $P_n(d)$
for the probability at being at depth $d$ after $n$ steps.
The correspondence between our approximation \eqref{eq:pnd} of
$P_n(d)$ on page \pageref{eq:pnd} and the empirical
distribution of search depths is shown in \cref{fig:depth}.

\begin{figure}
  \centering
  \begin{subfigure}[l]{0.49\linewidth}
    \includegraphics{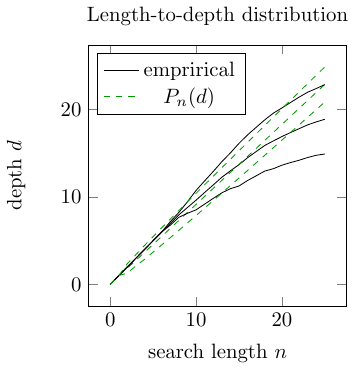}
  \end{subfigure}
  \begin{subfigure}[r]{0.49\linewidth}
    \includegraphics{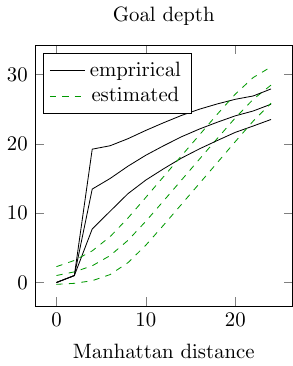}
  \end{subfigure}
  \caption{
    Left, the length-to-depth distribution $P_n(d)$ matched to the
    empirical depth-distribution.
    Right, the estimated goal distribution compared to the empirical goal
    distribution.
    All distributions are shown
    together with one standard deviation above and below.
    $P_n(d$ approximately matches the empirical distribution
    until about level 20, where the branching factor estimates
    ceases to be valid.
    The goal distribution only roughly matches the truth.
  }
  \label{fig:depth}
  \label{fig:goal-dist}
\end{figure}
\todo{maybe change starting node?}

\paragraph{Goal probability estimates}
A natural problem feature of N-Puzzle instances is the
Manhattan distance $\mh(v_0, v^*)$ between the starting node $v_0$
and the goal node $v^*$ \citep{Russell2010}.
An N-puzzle configuration can be represented with the coordinates for
the different tiles,
$v = \langle (x_0,y_0), (x_1,y_1),\dots,(x_n,y_n) \rangle$
where $(x_i,y_i)$ is the coordinates of tile $i$ and $i=0$ represents
the empty tile.
The Manhattan distance is then
\[
  \mh(v, u) = \sum_{i=0}^N |x_i^u-x_i^v| + |y_i^u-y_i^v|.
\]
Note that $\mh(v_0,v^*)$ needs to be divided by 2 in order
to be an admissible heuristic.

Investigating the correlation between the Manhattan distance $\mh$
and the actual distance $\dist$, we find that
$\E[\dist(v, u)\mid \mh(v, u) =m]\approx 1.5m$
and
$\Std(\dist(v, u)\mid \mh(v, u) =m)\approx 3.5$.
This gives us a mean goal level $\mu = 1.5\mh(v_0, v^*)$
and standard deviation $\sigma=3.5$.
We use a Gaussian-inspired goal probability vector $\pnp$
with $\pnpi_i = c\cdot e^{-\frac{(i-\mu)^2}{2\cdot \sigma^2}}/\bgdown^i$
where $c = 1/\sum_{i=-\infty}^{\infty}e^{-\frac{(i-\mu)^2}{2\cdot\sigma^2}}$ is a normalising constant
and $\bgdown^i$ is the expected number of nodes on level $i$ based on the
global branching factor $\bgdown$.

The theoretical goal distribution is matched against the true goal
distribution in \cref{fig:goal-dist}.

\paragraph{Search time estimates}

\begin{figure}
  \centering
  \includegraphics{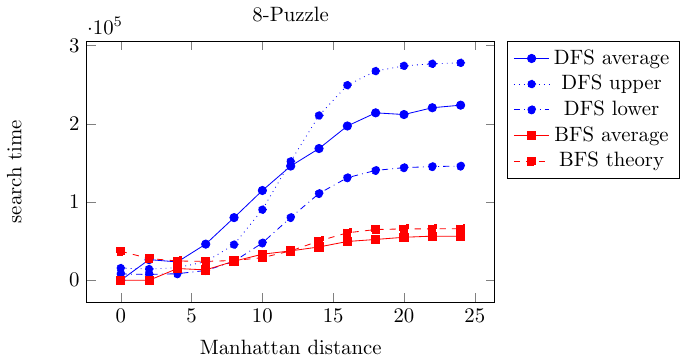}
  \caption{
    8-Puzzle search times for DFS tree search and BFS.
    The empirical averages are displayed together with the theoretical bounds for DFS,
    and the theoretical estimate for BFS.
    As expected, DFS tree search expands more nodes than BFS.
    Overall, there is a strong match between theory and practice, with average DFS
    search times generally being contained within the bounds, and BFS search times
    closely following their theoretical estimate.
    The empirical averages are based on a 100 runs per Manhattan distance.
  }
  \label{fig:onegoal}
\end{figure}


\begin{figure}
  \centering
  \includegraphics{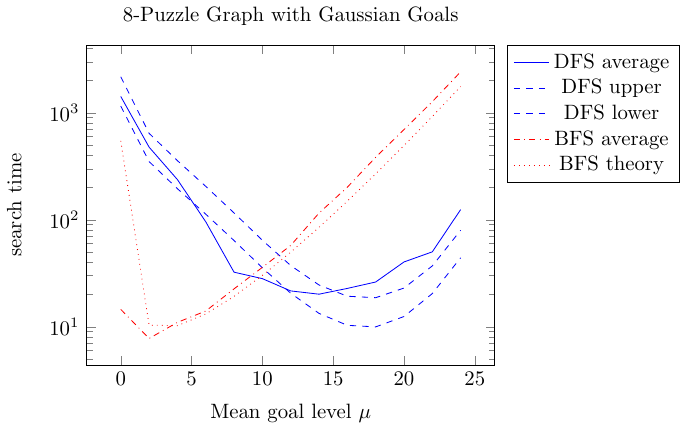}
  \caption{
    Search times for DFS tree search and BFS in an 8-Puzzle graph
    with nodes distributed according to the Gaussian
    goal probability vector
    $p_i = \min\left\{1, c\cdot e^{(i-\mu)^2/(2\sigma^2)} \right\}$
    where $c = \sum_{i=-\infty}^{\infty}e^{(i-\mu)^2/(2\sigma^2)}$ is a normalising constant.
    BFS is better than DFS for mean goal$\mu$ between 1 and 9,
    and DFS is better for $\mu$ between 10 and 20.
    The theoretical bounds slightly overestimate the search time
    of DFS for lower $\mu$, and slightly underestimates DFS
    search time for higher $\mu$,
    possibly as a result of the branching factors estimates
    being based on the middle levels of the graph.
    The empirical averages are based on a 100 runs per mean goal level.
  }
  \label{fig:50goals}
\end{figure}

\begin{figure}
  \centering
  \includegraphics{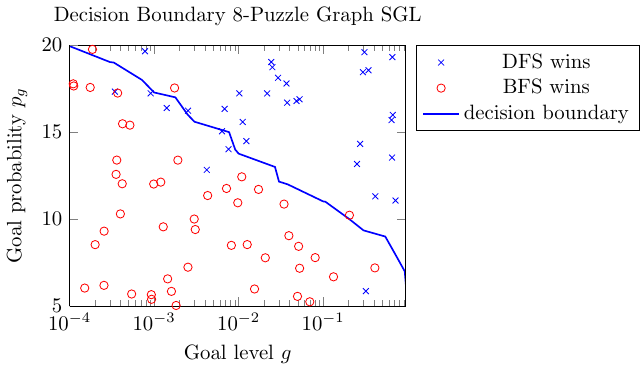}
  \caption{
    Empirical outcomes together with the theoretical decision boundary
    between DFS tree search and BFS in a modified 8-puzzle graph with
    a single goal level $g$ with goal probability $p_g$.
    As expected, DFS benefits by lower levels with higher goal probability.
    The decision boundary classifies $90\%$ of the points correctly.
  }
  \label{fig:n-puzzle-dec}
\end{figure}

We compare the search time estimates based on the above parameters
and the theory developed in \cref{sec:cb,sec:graph-param} with empirical
averages.
The results are displayed in \cref{fig:onegoal}.
Our averages are based on 100 randomly sampled problems of each occurring
Manhattan distance.
To avoid the changing dynamics of the lowest levels (see \cref{fig:branching-factors}),
we set the radius of search to 20.

As can be seen in \cref{fig:onegoal}, our theoretical
model predicts expected search time reasonably accurately.
We find it encouraging that our methods allow us to predict
the search time of especially DFS so well.
The theoretical estimates are off slightly for the levels where
the goal distribution is inaccurate (\cref{fig:goal-dist}).
To separate the sources of error, we also investigate two
8-Puzzle search problems with artificially sampled goals.
In these problems, we also take the opportunity to increase the
number of goals to give DFS a better chance in comparison.
In the first problem,
we used a Gaussian goal probability vector (\cref{fig:50goals}).
In the second problem, we used single goal level with varying
goal probability (\cref{fig:n-puzzle-dec}).

As expected, DFS beats BFS when many goals are located far from the start,
which is the case in the Gaussian model in \cref{fig:50goals}
with high $\mu$.
In the single goal level model high $g$ means goals located
far from the root, and high $p_g$ means high chance of random walking
into one.
The points of DFS takeover are well predicted by our theory
(\cref{fig:50goals,fig:n-puzzle-dec}).
In the original 8-puzzle, DFS struggles to random walk into the
single goal, and always needs to explore a substantial portion of
the graph in order to find a goal (note that \cref{fig:onegoal}
is not a logplot, as opposed to \cref{fig:50goals}).
Unsurprisingly, BFS graph search is virtually always faster than
DFS tree search in this setting.

\section{Adapting Results to Heuristic Search}
\label{sec:heuristic}

In many situations, heuristic search methods like A* or heuristic DFS
are better options than the uninformed methods of BFS and DFS discussed
in this paper.
In this section, we discuss how our results in previous
sections can be generalised to heuristic search.

\begin{definition}[Heuristic levels]
  Let $g: S\to\SetR$ be a \emph{consistent heuristic function},\footnote{
    A heuristic function is \emph{consistent} if it is admissible
    and satisfies the triangle inequality. See \citet[p.~95]{Russell2010}
    for details.
  }
  and let $f(v)\geq \dist(v_0,v)$ be the length of the current search
  path reaching $v$. Let $h(v) = f(v) + g(v)$.
  We define two generalisations of \cref{def:level}:
  Let the \emph{$g$-level $l$} be the set of nodes with $g(v)=l$,
  and let the \emph{$h$-level $l$} be the set of nodes with $h(v)=l$.
\end{definition}

A popular method for heuristic search is A*, which can be seen
as a generalisation of BFS.
The main difference between BFS and A* is that while BFS expands
the search graph according to levels, A* expands the graph according to
$h$-levels.
In analysing iterative deepening A* in trees, \citet{Korf2001} has
argued that the downwards branching factors remain the
same when considering the considering the tree layered by $g$-levels
instead of levels, and that the goal probability vector is shifted
by a constant $k$ depending on the heuristic
(so $p_g$ for BFS is $p_{g-k}$ for A*).
Other researchers prefer to model the effect of the heuristic as
reducing the branching factor \citep[p.~111]{Russell2010}.
It is an empirical question which model works best in our case.
Investigating this constitutes a promising line of future work.

\begin{algorithm}
  \begin{algorithmic}
    \State path $\gets$ empty list
    \State \Call{DFS-tree-rec}{$N$, $C$, start node, path, radius, $g$}
    \State
    \Function{Heuristic-DFS-rec}{$N$, $C$, $u$, path, radius, $g$}
    \State path.append($u$)
    \If{$C(u)$}
    \Return $u$
    \EndIf
    \If{length(path) $<$ radius}
    \State ranked-neighbours $\gets$ sort($N(u)\setminus$path, $g$)\Comment{with low $g$ first}
    \For{$v$ in ranked-neighbours}
    \State\Call{DFS-tree-rec}{$N$, $C$, $v$, path, radius, $g$}
    \EndFor
    \EndIf
    \EndFunction
  \end{algorithmic}
  \caption{Heuristic DFS tree search}
  \label{alg:heuristic-dfs}
\end{algorithm}

\begin{figure}
  \centering
  \includegraphics{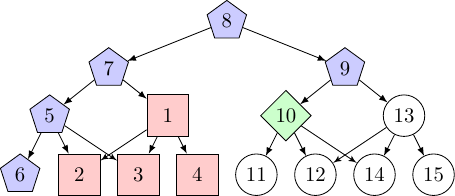}
  \caption{Heuristic DFS may be seen as a DFS search that follows $g$-levels instead
    of levels,
    and that possibly ``backtracks'' up above the initial $g$-level that it started at.}
  \label{fig:heuristic-dfs}
\end{figure}

Heuristically guided versions of DFS include Beam search and
heuristic DFS (see \cref{alg:heuristic-dfs}).
While A* follows $h$-levels, it is most natural to understand heuristic
DFS to follow $g$-levels
(see \cref{fig:heuristic-dfs}).
Heuristic DFS starts at some intermediate $g$-level.
Once it exhausts nodes on $g$-levels below it (assuming it does not
find a goal there), it may find its way to higher $g$-levels than it
started at.
This may be seen as a generalised notion of backtracking.
Extending our theory of DFS search time to heuristic DFS
would involve finding a theory for $g$-level branching factors,
and making the generalised notion of backtracking precise.

\section{Summary and Outlook}
\label{sec:conclusions}

Search and optimisation problems appear in different flavors
throughout the field of artificial intelligence;
in planning, problem solving, games, and learning.
Therefore even minor improvements to search performance
can potentially lead to gains in many aspects of intelligent
systems.
It is even possible to equate intelligence with (Bayesian expectimax)
optimisation performance \citep{Legg2007}.

\paragraph{Summary}
In this paper we have derived analytical
results for expected runtime of BFS and DFS.
\cref{sec:sgl,sec:mgl}
focused on BFS and DFS \emph{tree search} where explored nodes
were not remembered.
A vector $\p=[p_1,\dots,p_D]$ described \emph{a priori} goal probabilities
for the different levels of the tree.
This concrete but general model of goal distribution allowed us
to calculate approximate closed-form expression of both BFS and DFS
average runtime.
Earlier studies have only addressed \emph{worst case} runtimes:
For example \citet{Knuth1975} and followers for DFS;
\citet{Korf2001} and followers for IDA*,
a linear space version of BFS.

\cref{sec:cb} generalised the model of \cref{sec:sgl,sec:mgl}
to non-tree graphs.
In addition to the goal probability vector $\p$, the graph search
analysis required additional structural information in the form of a
length-to-depth counter $L$,
which was inferred from branching factors in \cref{sec:graph-param}.
The DFS graph search estimates also took the form of less precise bounds.
The analysis of \cref{sec:cb} does not supersede the analysis in \cref{sec:sgl,sec:mgl},
as the bounds of \cref{sec:cb} become uninformative when the graph
is a tree.
The analytical results are generally consistent with empirical outcomes.

In \cref{sec:heuristic} we also outlined how our results can be
extended to heuristic search.

\paragraph{The value of expected search time}

Several applications are naturally directed to maximising
expected utility, including games and reinforcement learning.
In such contexts, average performance is often more important
than worst case performance.
Indeed, in our model, worst case performance is always
$2^{D+1}$ for both BFS and DFS since it is not \emph{a priori} necessary that a
goal exists.
Our expected runtime estimates are much more informative.

Being able to estimate expected search time for BFS and DFS
is valuable for several reasons.
First, and most obvious, it can be used for allocating resources,
and in deciding whether a problem is approachable with BFS or DFS
at all.
Second, expected search time can guide the choice of
algorithm, and the choice of graph representation.
Choosing the best algorithm and the best representation
can improve performance substantially.
Third, the results also offer theoretical insight into BFS
and DFS.
As BFS and DFS are opposites, and in a sense are the most fundamental
ways to search, we have high hopes that our results and
techniques can be useful both in the construction of
new search algorithms, and in the analysis of existing ones.
For example, A* and IDA* may be viewed as generalisations of BFS,
and Beam Search
and Greedy Best-First as generalisations of DFS.
We find the DFS tree search results for graphs
developed in \cref{sec:cb,sec:graph-param} especially promising,
and believe they may find use outside the domain considered in this paper.

\section*{Acknowledgements}
Thanks to David Johnston and Aaron Stockdill
for proof reading drafts of this paper.


\pagebreak
\appendix

\section{List of notation}
\label{sec:list-of-notation}

\begin{longtable}{lp{0.85\textwidth}}
$P$
    & Probability\\
$X$, $Y$
    & Random variables\\
$\E[\,\cdot\,]$
    & Expectation of a random variable \\
$\tc(p,m)$
    & Expectation of a truncated geometric variable with parameters $p$ and $m$\\
$O$
    & Big-O notation\\
$\mathit{EC}$
    & Edge cost\\
$h$
    & Heuristic function\\
$g$
    & Accumulated path cost from start node \\
$Q$
    & Objective function\\
$D$
    & Maximum search depth/level\\
$D'$
    & Radius of search (maximum path length DFS search)\\
$p_g$
    & Goal probability at a single goal level $g$\\
$p_k$
    & Goal probability for a level $k$\\
$q_k$
    & $1-p_k$\\
$\p$
    & Vector of probabilities for multiple goal levels\\
$\mu, \sigma^2$
    & Goal peak and goal spread in Gaussian binary tree\\
$\Gamma$
    & Probability that a goal exists\\
$\Gamma_k$
    & Probability that level $k$ has a goal\\
$F_k$
    & Probability that level $k$ has the first goal\\
$\tbfss$, $\tdfss$
    & Expected BFS search time and approximate expected DFS search time
    in a complete tree with a single goal level\\
$\tbfsm$, $\tdfsm$
    & Expected BFS search time and approximate expected DFS search time
    in a complete tree with multiple goal levels\\
$\tbfsc$, $\tdfsc$
    & Expected BFS search time and approximate expected DFS search time
    in a graph with colliding branches\\
$\tbfsb$, $\tdfsb$
    & Expected BFS search time and approximate expected DFS search time
    in the binary grammar problem\\
$\tbfsf$, $\tdfsf$
    & Expected BFS search time and approximate expected DFS search time
      in the full grammar problem\\
  $\dist$
    & Distance (shortest path between two nodes)\\
  $\lvl$
    & Level (distance from start node)\\
$\delta_n$
    & The first node on level $n$ reached by DFS\\
  $K(l, d, r)$, $\bar K(l, d, r)$
    & Depth-to-depth counters, counting the number of (unique and non-unique)
      level $d$ descendants
      are reachable from level $l$ in at most $r$ steps\\
$L(n,d)$, $\bar L(n, d)$
    & Length-to-depth counters, counting the number of (unique and non-unique)
      level $d$ descendants are
      reachable after an average $n$ step path (i.e.\ from $\delta_n$)\\
$\lfg$, $\lbg$
    & Length-to-depth counts for the binary grammar problem
      and the full grammar problem\\
  $l, d$
    & Level/depth in graph\\
  $n$
  & Path length for search in graph\\
$A_{n,d}$
    & Number of nodes reachable from $\delta_n$ not reachable from $\delta_{n+1}$\\
$S_n$
    & Descendants of $\delta_n$\\
$T_n$
    & Descendants of $\delta_n$ that are not descendants of $\delta_{n+1}$\\
$U_n$
    & The number of nodes above level $n$.\\
$\tau_n$
    & The probability that $T_n$ contains a goal (\cref{le:tn-goal})\\
$\phi_n$
    & The probability that $T_n$ inhabits the first goal\\
$b$
    & Branching factor trees\\
$\bl_{\dir}$, $\bgdir$
    & Local and global branching factors in graphs for $\dir\in\{\up,\side,\down\}$\\
$p_{\dir}$
    & Average probability of moving on level in direction $\dir$\\
$p_{\dir_1,\dir_2}$
    & Probability of moving one level in direction $\dir_2$ given came
      from direction $\dir_2$\\
  $\pnp$
  & Goal probability vector for the 8-puzzle\\
  $P_n(d)$
    & Probability at being at depth $d$ after travelling $n$ steps\\
$\epsilon$
    & Empty string\\
\end{longtable}

\end{document}